\documentclass[10pt,twocolumn,letterpaper]{article}

\usepackage{cvpr}
\usepackage{times}
\usepackage{epsfig}
\usepackage{graphicx}
\usepackage{amsmath}
\usepackage{amssymb}

\usepackage{lmodern}
\usepackage{enumerate}
\usepackage{amsthm}
\usepackage[usenames]{color}
\usepackage{capt-of}
\usepackage{graphicx} 
\usepackage{epstopdf}
\usepackage{wrapfig}
\usepackage[font=scriptsize,labelfont=bf]{caption}
\usepackage{multirow}
\usepackage{caption}
\usepackage{algorithm}
\usepackage{algorithmic}

\usepackage{subfig}
\usepackage{tabularx}
\usepackage{afterpage}
\usepackage{floatflt}
\usepackage{arydshln}
\usepackage{multirow, makecell} 

\usepackage{numprint}
\makeatletter
\@namedef{ver@everyshi.sty}{}
\makeatother
\usepackage{tikz}
\usepackage[english]{babel}

\newtheorem{lemma}{Lemma}[section]

\newtheorem{theorem}{Theorem}[section]

\newtheorem{definition}{Definition}[section]

\newcommand{\inner}[1]{\left\langle#1\right\rangle}

\def\Pr{\mathrm{P}}

\def\R{\mathbb{R}}

\def\N{\mathbb{N}}

\newcommand{\norm}[1]{\left\|#1\right\|}

\def\Pr{\mathrm{P}}

\def\Exp{\mathbb{E}}

\def\argmax{\mathop{\rm arg\,max}\limits}
\def\minop{\mathop{\rm min}\limits}
\def\maxop{\mathop{\rm max}\limits}
\def\sign{\mathop{\rm sign}\limits}

\def\min{\mathop{\rm min}\nolimits}
\def\max{\mathop{\rm max}\nolimits}

\newif\ifpaper
\papertrue

\usepackage[pagebackref=true,breaklinks=true,letterpaper=true,colorlinks,bookmarks=false]{hyperref}

 \cvprfinalcopy 

\def\cvprPaperID{4863} 

\begin{document}

\title{Why ReLU networks yield high-confidence predictions far away from
the training data and how to mitigate the problem}

\author{Matthias Hein\\
	University of T{\"u}bingen
	\and
	Maksym Andriushchenko\\
	Saarland University
	\and
	Julian Bitterwolf\\
	University of T{\"u}bingen
}


\maketitle
\thispagestyle{empty}

\begin{abstract}
Classifiers used in the wild, in particular for safety-critical systems, 
should not only have good generalization properties but also should know when they don't know, in particular make 
low confidence predictions far away from the training data. We show that
ReLU type neural networks which yield a piecewise
linear classifier function fail in this regard as they produce almost always high confidence predictions far away from the training
data. For bounded domains like images we propose a new robust
optimization technique similar to adversarial training which enforces low confidence
predictions far away from the training data. We show that this technique is surprisingly effective in reducing
the confidence of predictions far away from the training data 
while maintaining high confidence predictions and test error on the original classification task
compared to standard training.
\end{abstract}

\section{Introduction}\label{sec:intro}

Neural networks have recently obtained state-of-the-art performance in several application domains like object recognition and speech recognition. They have become the 
de facto standard for many learning tasks. Despite this great success story and very good prediction performance there are also aspects of neural networks which are undesirable.
One property which is naturally expected from any classifier is that it should know when it does not know or said more directly: far away from the training data a classifier
should not make high confidence predictions. This is particularly important in safety-critical applications like autonomous driving or medical diagnosis systems where such an 
input should either lead to the fact that other redundant sensors are used or that a human doctor is asked to check the diagnosis. It is thus an important property of a classifier which however has not received much attention despite the fact that it seems to be a minimal requirement for any classifier.

There have been many cases reported where high confidence predictions are made far away from the training data by neural networks, e.g. on fooling images \cite{NguYosClu2015}, for out-of-distribution images \cite{HenGim2017} or in a medical diagnosis task \cite{LeiEtAl2017}. Moreover, it has been observed that,
even on the original task, neural networks often produce overconfident predictions \cite{GuoEtAl2017}. A related but different problem are adversarial samples where very small modifications of the input can change the classifier decision \cite{SzeEtAl2014, GooShlSze2015, MooFawFro2016}. Apart from methods which provide robustness guarantees 
for neural networks \cite{HeiAnd2017,WonKol2018,RagSteLia2018,MirGehVec2018} which give still only reasonable guarantees for small networks, up to our knowledge the only approach which has not been broken again \cite{CarWag2016,CarWag2017,AthCarWag2018} is adversarial training \cite{MadEtAl2018} using robust optimization techniques.

While several methods have been proposed to adjust overconfident predictions on the true input distribution using softmax calibration \cite{GuoEtAl2017}, ensemble techniques \cite{LakEtAl2017}
or uncertainty estimation using dropout \cite{GalGha2016}, only recently the detection of out-of-distribution inputs \cite{HenGim2017} has been tackled. The existing approaches basically either use adjustment techniques of the softmax outputs \cite{DeVTay2018,LiaLiSri2018} by temperature rescaling \cite{GuoEtAl2017} or they use a generative model like a VAE or GAN to model boundary inputs of the true distribution \cite{LeeEtAl2018,WanEtAl2018} in order to discriminate in-distribution from out-of-distribution inputs directly in the training process. While all these approaches are significant steps towards obtaining more reliable classifiers, the approaches using a generative model have been recently challenged by 
\cite{NalEtAl2018,HenMazDie2019} which report that generative approaches can produce highly confident density estimates for inputs outside of the class they are supposed to model. Moreover, note that
the quite useful models for confidence calibration on the input distribution like \cite{GalGha2016,GuoEtAl2017,LakEtAl2017} cannot be used for out-of-distribution detection as it has been observed in \cite{LeiEtAl2017}. Another approach is the introduction of a rejection option into the classifier \cite{TewBar2007, BenBou2016}, in order to avoid decisions  the classifier
is not certain about.

In this paper we will
show that for the class of ReLU networks, that are networks with fully connected, convolutional and residual layers, where just ReLU or leaky ReLU are used as activation functions and 
max or average pooling for convolution layers, basically any neural network which results in a piecewise affine classifier function, produces arbitrarily high confidence predictions far away from the training data. This implies that techniques which operate on the output of the classifier cannot identify these inputs as out-of-distribution inputs. 
On the contrary we formalize the well known fact that RBF networks produce almost uniform confidence over the classes far away from the training data, which shows that there exist classifiers which satisfy the minimal requirement of not being confident in areas where one has never seen data.
Moreover, we propose a robust optimization scheme motivated by adversarial training \cite{MadEtAl2018} which simply enforces uniform confidence predictions on noise images which are by construction
far away from the true images. We show that our technique not only significantly reduces confidence on such noise images, but also on other unrelated image classification tasks and
in some cases even for adversarial samples generated for the original classification task. The training procedure is simple, needs no adaptation for different out-of-distribution tasks, has similar complexity as standard adversarial training and achieves similar performance on the original classification
task.

\section{ReLU networks produce piecewise affine functions}\label{sec:explicit}
We quickly review in this section the fact that ReLU networks lead to continuous piecewise affine classifiers, see \cite{AroEtAl2018,CroHei18}, which we briefly summarize
in order to set the ground for our main theoretical result in Section \ref{sec:faraway}. 
\begin{definition} A function $f:\R^d \rightarrow \R$ is called \emph{piecewise affine} if there exists a finite set of polytopes $\{Q_r\}_{r=1}^M$ 
	(referred to as \emph{linear regions} of $f$) such that $\cup_{r=1}^M Q_r= \R^d$ and $f$ is an affine function when restricted to  every $Q_r$.
\end{definition}
Feedforward neural networks which use piecewise affine activation functions (e.g. ReLU, leaky ReLU) and are linear in the output layer can be rewritten as continuous piecewise affine functions \cite{AroEtAl2018}. This includes fully connected, convolutional, residual layers and even skip connections as all these layers are just linear mappings. Moreover, it includes further average pooling and max pooling.  More precisely, the classifier is a function $f:\R^d \rightarrow \R^K$, where $K$ are the number of classes, such that each component $f_i:\R^d \rightarrow \R$, is a continuous piecewise affine function and the $K$ components $(f_i)_{i=1}^K$ have the same set of linear regions. Note that  explicit upper bounds on the number of linear regions have been given \cite{MonEtAl2014}.

In the following we follow \cite{CroHei18}. For simplicity we just present fully connected layers (note that convolutional layers are a particular case of them). Denote by $\sigma:\R \rightarrow \R$, $\sigma(t)=\max\{0,t\}$, the ReLU activation function, by $L+1$ the number of layers and $W^{(l)} \in\mathbb{R}^{n_l \times n_{l-1}}$ and $b^{(l)} \in \R^{n_l}$ respectively are the weights and offset vectors of layer $l$,  for $l=1,\ldots,L+1$ and $n_0=d$. For $x\in \R^d$  one defines $g^{(0)}(x)=x$. Then one can recursively define the pre- and post-activation output of every layer as
\begin{gather*}    f^{(k)}(x)=W^{(k)}g^{(k-1)}(x)+b^{(k)}, \quad \mathrm{ and }\\g^{(k)}(x)=\sigma(f^{(k)}(x)), \quad k=1,\ldots,L,\end{gather*}
so that the resulting classifier is obtained as $f^{(L+1)}(x)=W^{(L+1)}g^{(L)}(x) + b^{(L+1)}$.

Let $\Delta^{(l)},\Sigma^{(l)}\in \R^{n_l \times n_l}$ for $l=1,\ldots,L$ be diagonal matrices defined elementwise as
\begin{gather*} \Delta^{(l)}(x)_{ij} = \begin{cases} \sign(f_i^{(l)}(x)) & \textrm{ if } i=j,\\ 0 & \textrm{ else.} \end{cases}, \\ 
\Sigma^{(l)}(x)_{ij} = \begin{cases} 1 & \textrm{ if } i=j \textrm{ and } f_i^{(l)}(x)>0,\\ 0 & \textrm{ else.} \end{cases}.\end{gather*}
Note that for leaky ReLU the entries would be $1$ and $\alpha$ instead.
This allows to write $f^{(k)}(x)$ as composition of affine functions, that is
\[\begin{split} f^{(k)}(x)=&W^{(k)}\Sigma^{(k-1)}(x)\Big(W^{(k-1)} \Sigma^{(k-2)}(x)\\ & \times \Big(\ldots  \Big( W^{(1)}x + b^{(1)}\Big) \ldots\Big)+ b^{(k-1)} \Big) + b^{(k)}, \end{split}\]
We can further simplify the previous expression as $f^{(k)}(x) = V^{(k)}x + a^{(k)}$, with $V^{(k)} \in \R^{n_k \times d}$ and $a^{(k)} \in \R^{n_k}$ given by
\begin{gather*} V^{(k)} = W^{(k)}\Big( \prod_{l=1}^{k-1} \Sigma^{(k-l)}(x)W^{(k-l)}\Big) \quad \mathrm{ and }\\ a^{(k)} = b^{(k)} + \sum_{l=1}^{k-1} \Big(\prod_{m=1}^{k-l} W^{(k+1-m)} \Sigma^{(k-m)}(x)\Big) b^{(l)}. \end{gather*}
The polytope $Q(x)$, the linear region containing $x$, can be characterized as an intersection of $N=\sum_{l=1}^L n_l$ half spaces given by 
\[\Gamma_{l,i}=\big\{z \in \R^d \,\Big|\, \Delta^{(l)}(x)\big(V_i^{(l)}z +a_i^{(l)}\big)\geq 0\big\},\]
for $l=1,\ldots,L$, $i=1,\ldots,n_l$, namely
\[Q(x)=\bigcap_{l=1,\ldots,L}\bigcap_{i=1,\ldots,n_l} \Gamma_{l,i}. \]
Note that $N$ is also the number of hidden units of the network.
Finally, we can write 
\[ \left.f^{(L+1)}(z)\right|_{Q(x)}=V^{(L+1)}z + a^{(L+1)},\] which is the affine restriction of $f$ to $Q(x)$.


\section{Why ReLU networks produce high confidence predictions far away from the training data}\label{sec:faraway}
With the explicit description of the piecewise linear classifier resulting from a ReLU type network from Section \ref{sec:explicit}, we can now formulate our main theorem. It shows that, as long a very mild condition
on the network holds, for any $\epsilon>0$ one can always find for (almost) \textbf{all} directions an input $z$ far away from the training data which realizes a confidence of $1-\epsilon$ on $z$ for a certain class.
However, before we come to this result, we first present a technical lemma needed in the proof, which uses that all linear regions are polytopes and thus convex sets.
\begin{lemma}\label{le:rays}
Let $\{Q_i\}_{l=1}^R$ be the set of linear regions associated to the ReLU-classifier $f:\R^d \rightarrow \R^K$. For any $x \in \R^d$
there exists $\alpha \in \R$ with $\alpha>0$ and $t \in \{1,\ldots,R\}$ such that $\beta x \in Q_t$ for all $\beta \geq \alpha$.
\end{lemma}

All the proofs can be found in the appendix.
Using Lemma \ref{le:rays} we can now state our first main result.
\begin{theorem}\label{th:relu}
Let $\R^d=\cup_{l=1}^R Q_l$ and $f(x)=V^l x + a^l$  be the piecewise affine representation of the
output of a ReLU network on $Q_l$. Suppose that $V^l$ does not contain identical rows for all $l=1,\ldots,R$, then for almost any $x \in \R^d$ and $\epsilon>0$
there exists an $\alpha>0$ and a class $k \in \{1,\ldots,K\}$ such that for $z=\alpha x$ it holds
\[ \frac{e^{f_k(z)}}{\sum_{r=1}^K e^{f_r(z)}} \geq 1-\epsilon.\]
Moreover, $\lim\limits_{\alpha \rightarrow \infty}  \frac{e^{f_k(\alpha x)}}{\sum_{r=1}^K e^{f_r(\alpha x)}} =1$.
\end{theorem}

\begin{figure}
\begin{center}
\includegraphics[width=0.4\textwidth]{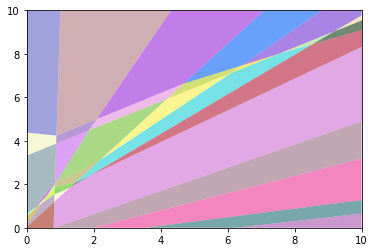}
\caption{\label{fig:regions}A decomposition of $\R^2$ into a finite set of polytopes for a two-hidden layer ReLU network. The outer polytopes extend to infinity. This is where ReLU networks realize
arbitrarily high confidence predictions. The picture is produced with the code of \cite{jordan2019provable}.}
\end{center}
\end{figure}
Please note that the condition that for a region the linear part $V^l$ need not contain two identical rows is very weak. It is hardly imaginable that this is ever true for a normally trained network unless the output of the network is constant anyway. Even if it is true, it just invalidates the assertion of the theorem for the points lying in this region.
Without explicitly enforcing this condition it seems impossible that this is true for all possible asymptotic regions extending to infinity (see Figure \ref{fig:regions}). However, it is also completely open how this condition could be enforced during training of the network.

The result implies that for ReLU networks there exist infinitely many inputs which realize arbitrarily high confidence predictions of the networks. It is easy to see that the temperature rescaling of the softmax, $\frac{e^{f_k(x)/T}}{\sum_{l=1}^K e^{f_l(x)/T}}$, for temperature $T>0$, as used in \cite{LiaLiSri2018}, will not be able to detect these cases,
in particular since the first step of the method in \cite{LiaLiSri2018} consists of going in the direction of increasing confidence. Also it is obvious that using a reject option in the classifier, see e.g. \cite{BarWeg2008}, will not help to detect these instances either. The result is negative in the sense that it looks like that without modifying the architecture of a ReLU network it is impossible to prevent this phenomenon. Please note that from the point of view of Bayesian decision theory the softmax function is the correct transfer function \cite{LapHeiSch2016} for the cross-entropy loss turning the classifier output $f_k(x)$ into an estimate $\Pr(Y=k\,|x,f)=\frac{e^{f_k(x)}}{\sum_{l=1}^K e^{f_l(x)}}$ for the conditional probability at $x$.

While the previous result seems not to be known, the following result is at least qualitatively known \cite{GooShlSze2015} but we could not find a reference for it. In contrast to the ReLU networks
it turns out that Radial Basis Function (RBF) networks have the property to produce approximately uniform confidence predictions far away from the training data. Thus there exist classifiers which satisfy the minimal requirement which we formulated in Section \ref{sec:intro}. In the following theorem we explicitly quantify what ``far away'' means in terms of parameters of the RBF classifier and the training data.
\begin{theorem}\label{th:RBF}
Let $f_k(x)=\sum_{l=1}^N \alpha_{kl} e^{-\gamma \norm{x-x_l}^2_2}$, $k=1,\ldots,K$ be an RBF-network trained with cross-entropy loss on the training data $(x_i,y_i)_{i=1}^N$. We define
$r_{\min}=\minop_{l=1,\ldots,N} \norm{x-x_l}_2$ and $\alpha=\maxop_{r,k} \sum_{l=1}^N |\alpha_{rl}-\alpha_{kl}|$. If $\epsilon>0$ and 
\[ r^2_{\min} \geq \frac{1}{\gamma}\log\Big(\frac{\alpha}{\log(1+K\epsilon)}\Big),\]
then for all $k=1,\ldots,K$,
\[ \frac{1}{K}-\epsilon \; \leq\; \frac{e^{f_k(x)}}{\sum_{r=1}^K e^{f_r(x)}} \; \leq \; \frac{1}{K}+\epsilon.\] 
\end{theorem}

We think that it is a very important open problem to realize a similar result as in Theorem \ref{th:RBF} for a class of neural networks. Note that arbitrarily high confidence predictions
for ReLU networks can be obtained only if the domain is unbounded, e.g. $\R^d$. However, images are contained in $[0,1]^d$ and thus Theorem \ref{th:relu} does not directly apply, even
though the technique can in principle be used to produce high-confidence predictions (see Table \ref{tab:alpha_main}). In the next section we propose a novel training scheme enforcing low confidence predictions on inputs
far away from the training data.

\section{Adversarial Confidence Enhanced Training}\label{sec:conftraining}
In this section we suggest a simple way to adjust the confidence estimation of a neural network far away from the training data, not necessarily restricted to ReLU networks studied in Theorem \ref{th:relu}. Theorem \ref{th:relu} tells us that for ReLU networks a post-processing of the softmax scores is not sufficient to avoid high-confidence predictions far away from the training data - instead there seem to be two potential ways to tackle the problem: a) one uses an extra generative model either for the in-distribution or for the out-distribution or b) one modifies directly the network via an adaptation of the training process so that uniform confidence predictions are enforced far away from the training data. As recently problems with generative models have been pointed out which assign high confidence to samples from the out-distribution \cite{NalEtAl2018} and thus a) seems less promising, we explore approach b).

We assume that it is possible to characterize a distribution of data points $p_{\textrm{out}}$ on the input space for which we are sure that they do not belong to the true distribution $p_{\textrm{in}}$ resp. the set of the intersection of their supports has zero or close to zero probability mass. An example of such an out-distribution $p_{\textrm{out}}$ would be the uniform distribution on $[0,1]^{w \times h}$ ($w \times h$ gray scale images) or similar noise distributions. Suppose that the in-distribution consists of certain image classes like handwritten digits, then the probability mass of all images of handwritten digits under the $p_{\textrm{out}}$ is zero (if it is really a low-dimensional manifold) or close to zero.

In such a setting the training objective can be written as a sum of two losses:
\begin{align}\label{eq:CEDA}
 \frac{1}{N}\sum_{i=1}^N L_{CE}(y_i,f(x_i)) + \lambda\, \Exp\big[L_{p_{\textrm{out}}}(f,Z)\big],
\end{align}
where $(x_i,y_i)_{i=1}^N$ is the i.i.d. training data, 
$Z$ has distribution $p_{\textrm{out}}$ and
\begin{align}
L_{CE}(y_i,f(x_i)) &= -\log\Big(\frac{e^{f_{y_i}(x_i)}}{\sum_{k=1}^K e^{f_k(x_i)}}\Big) \\
L_{p_{\textrm{out}}}(f,z)&= \maxop_{l=1,\ldots,K} \log\Big(\frac{e^{f_l(z)}}{\sum_{k=1}^K e^{f_k(z)}}\Big).\label{eq:Lpout}
\end{align} 
$L_{CE}$ is the usual cross entropy loss on the original classification task and $L_{p_{\textrm{out}}}(f,z)$ is the maximal log confidence over all classes,
where the confidence of class $l$ is given by $\frac{e^{f_l(z)}}{\sum_{k=1}^K e^{f_k(z)}}$, with the softmax function as the link function. The full loss can be easily minimized by using SGD with batchsize $B$ for the original data and adding $\lceil \lambda B \rceil$ samples from $p_{\textrm{out}}$ on which one enforces a uniform distribution over the labels. We call this process in the following \emph{confidence enhancing data augmentation (CEDA)}. We note that in a concurrent paper \cite{HenMazDie2019} a similar scheme has been proposed, where they use as $p_{\textrm{out}}$ existing large image datasets, whereas we favor an agnostic approach where $p_{\textrm{out}}$ models a certain ``noise'' distribution on images.

The problem with CEDA is that it might take too many samples to enforce low confidence on the whole out-distribution. Moreover, it has been shown in the area of adversarial manipulation that data augmentation is not sufficient for robust models and we will see in Section \ref{sec:exp} that indeed CEDA models still produce high confidence predictions in a neighborhood of noise images. Thus, we propose to use ideas from robust optimization similar to adversarial training which \cite{SzeEtAl2014, GooShlSze2015,MadEtAl2018} apply to obtain robust networks against adversarial manipulations. Thus we are enforcing low confidence not only at the point itself but actively minimize the worst case in a neighborhood of the point. This leads to the following formulation of \emph{adversarial confidence enhancing training (ACET)}
\begin{align}\label{eq:ACET}
 \frac{1}{N}\sum_{i=1}^N L_{CE}(y_i,f(x_i)) + \lambda \,\Exp\big[\maxop_{\norm{u-Z}_p \leq \epsilon} L_{p_{\textrm{out}}}(f,u)\big],
\end{align}
where in each SGD step one solves (approximately) for a given $z \sim p_{\textrm{out}}$ the optimization problem: 
\begin{equation}\label{eq:adv}
 \maxop_{\norm{u-z}_p \leq \epsilon} L_{p_{\textrm{out}}}(f,u).
\end{equation}
In this paper we use always $p=\infty$. Note that if the distributions $p_{\textrm{out}}$ and $p_{\textrm{in}}$ have joint support, the maximum in \eqref{eq:adv} could be obtained at a point in the support of the true distribution. However, if $p_{\textrm{out}}$ is a generic noise distribution like uniform noise or a smoothed version of it, then the number of cases where this happens has probability mass close to zero under $p_{\textrm{out}}$ and thus does not negatively influence in \eqref{eq:ACET} the loss $L_{CE}$ on the true distribution.
The optimization of ACET in \eqref{eq:ACET} can be done using an adapted version of the PGD method of \cite{MadEtAl2018} for adversarial training where one performs projected gradient descent (potentially for a few restarts) and uses the $u$ realizing the worst loss for computing the gradient. The resulting samples are more informative and thus lead to a faster and more significant reduction of
high confidence predictions far away from the training data. We use $\epsilon=0.3$ for all datasets. We present in Figure \ref{fig:advsamplesGray}
and \ref{fig:advsamplesColor} for MNIST and CIFAR-10 a few noise images together with their adversarial modification $u$ generated by applying PGD to solve \eqref{eq:adv}. One can observe that the generated images have no structure resembling images from the in-distribution.

\def\mnistpairs{05, 06, 07, 11, 15, 18, 34, 39}
\newcounter{mnistpairscounter}
\begin{figure*}
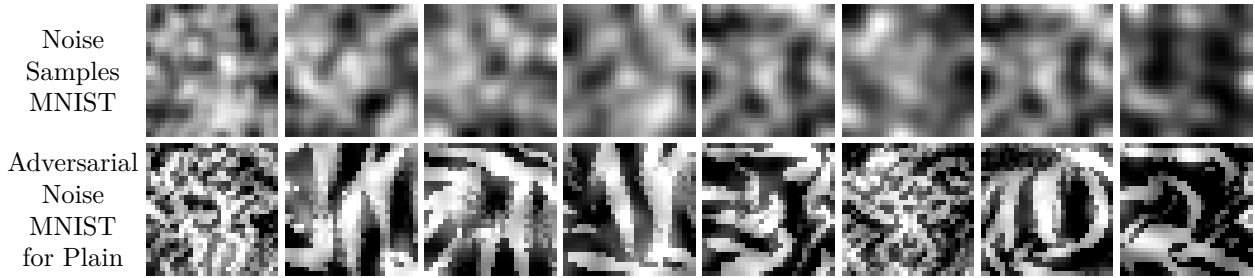

\begin{tikzpicture}
    \node[text width=2cm, align=center] at (-1.85, 0) {Noise Samples MNIST};
    \node[text width=2cm, align=center] at (-1.85, -1.85) {Adversarial Noise MNIST for Plain};
    \foreach \x in \mnistpairs    
    {    \node[inner sep=0pt] (s0) at (1.85*\value{mnistpairscounter},0)
            {\includegraphics[width=1.75cm]{noise_folder_plainm_\x x_in.png}};
         \node[inner sep=0pt] (a0) at (1.85*\value{mnistpairscounter},-1.85)
            {\includegraphics[width=1.75cm]{noise_folder_plainm_\x.png}};
        \stepcounter{mnistpairscounter}
    }
\end{tikzpicture}
   \caption{\label{fig:advsamplesGray}Top row: our generated noise images based on uniform noise resp. permuted MNIST together with a Gaussian filter and contrast rescaling. 
   Bottom row: for each noise image from above we generate the corresponding adversarial noise image using PGD with 40 iterations maximizing the second part of the loss in ACET
   for the plain model. 
   Note that neither in the noise images nor in the adversarially modified ones there is structure similar to a MNIST image. For ACET and CEDA it is very difficult to generate
   adversarial noise images for the fully trained models thus we omit them.
}
\end{figure*}

\def\cifarpairs{26, 00, 01, 03, 16, 18, 20, 24}
\newcounter{cifarpairscounter}
\begin{figure*}
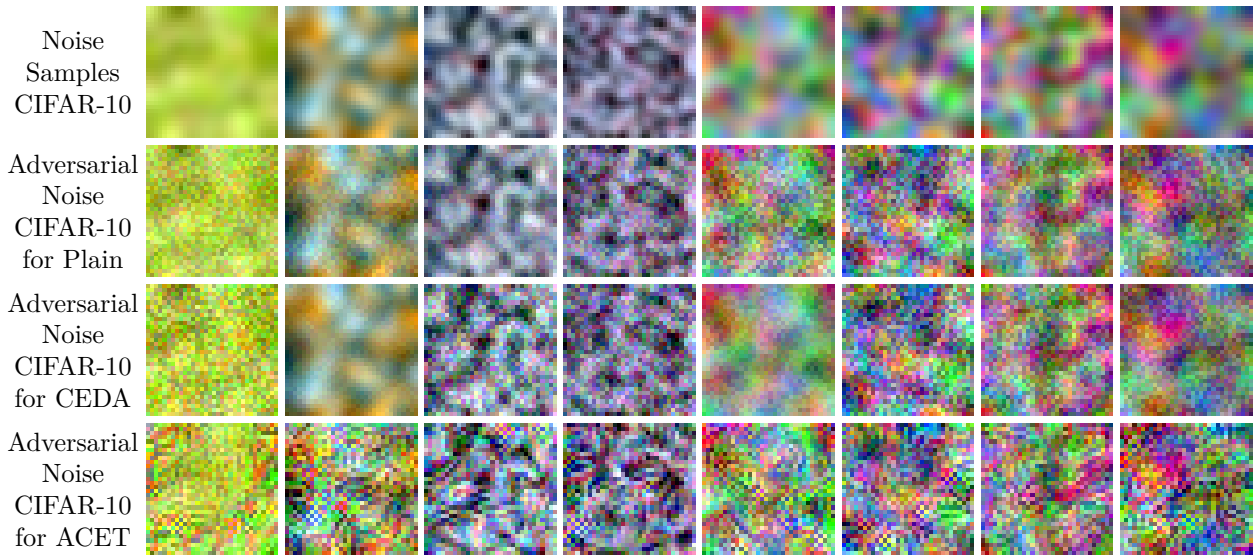

\begin{tikzpicture}
    \node[text width=2cm, align=center] at (-1.85, 0) {Noise Samples CIFAR-10};
    \node[text width=2cm, align=center] at (-1.85, -1.85) {Adversarial Noise CIFAR-10 for Plain};
    \node[text width=2cm, align=center] at (-1.85, -3.7) {Adversarial Noise CIFAR-10 for CEDA};
    \node[text width=2cm, align=center] at (-1.85, -5.55) {Adversarial Noise CIFAR-10 for ACET};
    
    \foreach \x in \cifarpairs    
   {    
            \node[inner sep=0pt] (s0) at (1.85*\value{cifarpairscounter},0)
            {\includegraphics[width=1.75cm]{noise_folder_cedac10_\x x_in.png}};
              \node[inner sep=0pt] (s0) at (1.85*\value{cifarpairscounter},-1.85)
            {\includegraphics[width=1.75cm]{noise_folder_plainc10_\x.png}};
         \node[inner sep=0pt] (a0) at (1.85*\value{cifarpairscounter},-3.7)
            {\includegraphics[width=1.75cm]{noise_folder_cedac10_\x.png}};
         \node[inner sep=0pt] (a0) at (1.85*\value{cifarpairscounter},-5.55)
            {\includegraphics[width=1.75cm]{noise_folder_acetc10_\x.png}};
        \stepcounter{cifarpairscounter}
    }
\end{tikzpicture}
  \caption{\label{fig:advsamplesColor}
  	Top row: our generated noise images based on uniform noise resp. permuted MNIST together with a Gaussian filter and contrast rescaling (similar to Figure \ref{fig:advsamplesGray}). Bottom rows: the corresponding adversarial images for the plain, CEDA, and ACET models. Neither the noise nor the adversarial noise images show similarity to CIFAR-10 images. 
}
\end{figure*}

\begin{table*}
	\begin{center}
		\begin{tabular}{|l||r|r|r||r|r|r||r|r|r|}
			\hline\multirowcell{2}{Trained on \\ \textbf{MNIST}} &
			\multicolumn{3}{c||}{Plain (TE: \textbf{0.51\%})}  &
			\multicolumn{3}{c||}{CEDA (TE: 0.74\%)} &
			\multicolumn{3}{c|}{ACET (TE: 0.66\%)} \\
			& MMC    & AUROC    & FPR@95    & MMC    & AUROC  & FPR@95    & MMC    & AUROC    &FPR@95 \\ \hline \hline
			MNIST
			& \textbf{0.991} & -- & -- & 0.987 & -- & -- & 0.986 & -- & -- \\ \hline
			FMNIST
			& 0.654 & 0.972 & 0.121
			& 0.373 & 0.994 & 0.027
			& \textbf{0.239} & \textbf{0.998} & \textbf{0.003}
			\\ \hline
			EMNIST
			& 0.821 & 0.883 & 0.374
			& 0.787 & 0.895 & 0.358
			& \textbf{0.752} & \textbf{0.912} & \textbf{0.313}
			\\ \hline
			grayCIFAR-10
			& 0.492 & 0.996 & 0.003
			& 0.105 & \textbf{1.000} & \textbf{0.000}
			& \textbf{0.101} & \textbf{1.000} & \textbf{0.000}
			\\ \hline
			Noise
			& 0.463 & 0.998 & 0.000
			& \textbf{0.100} & \textbf{1.000} & \textbf{0.000}
			& \textbf{0.100} & \textbf{1.000} & \textbf{0.000}
			\\ \hline
			Adv. Noise
			& 1.000 & 0.031 & 1.000
			& \textbf{0.102} & \textbf{0.998} & \textbf{0.002}
			& 0.162 & 0.992 & 0.042
			\\ \hline
			Adv. Samples
			& 0.999 & 0.358 & 0.992
			& 0.987 & 0.549 & 0.953
			& \textbf{0.854} & \textbf{0.692} & \textbf{0.782}
			\\ \hline
			\hline\multirowcell{2}{Trained on \\ \textbf{SVHN}} &
			\multicolumn{3}{c||}{Plain (TE: \textbf{3.53\%})}  &
			\multicolumn{3}{c||}{CEDA (TE: 3.50\%)} &
			\multicolumn{3}{c|}{ACET (TE: 3.52\%)} \\
			& MMC    & AUROC    & FPR@95    & MMC    & AUROC  & FPR@95    & MMC    & AUROC    &FPR@95 \\ \hline \hline
			SVHN
			& \textbf{0.980} & -- & -- & 0.977 & -- & -- & 0.978 & -- & -- \\ \hline
			CIFAR-10
			& 0.732 & 0.938 & 0.348
			& 0.551 & 0.960 & 0.209
			& \textbf{0.435} & \textbf{0.973} & \textbf{0.140}
			\\ \hline
			CIFAR-100
			& 0.730 & 0.935 & 0.350
			& 0.527 & 0.959 & 0.205
			& \textbf{0.414} & \textbf{0.971} & \textbf{0.139}
			\\ \hline
			LSUN CR
			& 0.722 & 0.945 & 0.324
			& 0.364 & 0.984 & 0.084
			& \textbf{0.148} & \textbf{0.997} & \textbf{0.012}
			\\ \hline
			Imagenet-
			& 0.725 & 0.939 & 0.340
			& 0.574 & 0.955 & 0.232
			& \textbf{0.368} & \textbf{0.977} & \textbf{0.113}
			\\ \hline
			Noise
			& 0.720 & 0.943 & 0.325
			& \textbf{0.100} & \textbf{1.000} & \textbf{0.000}
			& \textbf{0.100} & \textbf{1.000} & \textbf{0.000}
			\\ \hline
			Adv. Noise
			& 1.000 & 0.004 & 1.000
			& 0.946 & 0.062 & 0.940
			& \textbf{0.101} & \textbf{1.000} & \textbf{0.000}
			\\ \hline
			Adv. Samples
			& 1.000 & 0.004 & 1.000
			& 0.995 & 0.009 & 0.994
			& \textbf{0.369} & \textbf{0.778} & \textbf{0.279}
			\\ \hline
			\hline\multirowcell{2}{Trained on \\ \textbf{CIFAR-10}} &
			\multicolumn{3}{c||}{Plain (TE: 8.87\%)}  &
			\multicolumn{3}{c||}{CEDA (TE: 8.87\%)} &
			\multicolumn{3}{c|}{ACET (TE: \textbf{8.44\%})} \\
			& MMC    & AUROC    & FPR@95    & MMC    & AUROC  & FPR@95    & MMC    & AUROC    &FPR@95 \\ \hline \hline
			CIFAR-10
			& \textbf{0.949} & -- & -- & 0.946 & -- & -- & 0.948 & -- & -- \\ \hline
			SVHN
			& 0.800 & 0.850 & 0.783
			& 0.327 & 0.978 & 0.146
			& \textbf{0.263} & \textbf{0.981} & \textbf{0.118}
			\\ \hline
			CIFAR-100
			& 0.764 & 0.856 & 0.715
			& \textbf{0.761} & \textbf{0.850} & 0.720
			& 0.764 & 0.852 & \textbf{0.711}
			\\ \hline
			LSUN CR
			& 0.738 & \textbf{0.872} & \textbf{0.667}
			& \textbf{0.735} & 0.864 & 0.680
			& 0.745 & 0.858 & 0.677
			\\ \hline
			Imagenet-
			& 0.757 & 0.858 & 0.698
			& 0.749 & 0.853 & 0.704
			& \textbf{0.744} & \textbf{0.859} & \textbf{0.678}
			\\ \hline
			Noise
			& 0.825 & 0.827 & 0.818
			& \textbf{0.100} & \textbf{1.000} & \textbf{0.000}
			& \textbf{0.100} & \textbf{1.000} & \textbf{0.000}
			\\ \hline
			Adv. Noise
			& 1.000 & 0.035 & 1.000
			& 0.985 & 0.032 & 0.983
			& \textbf{0.112} & \textbf{0.999} & \textbf{0.008}
			\\ \hline
			Adv. Samples
			& 1.000 & 0.034 & 1.000
			& 1.000 & 0.014 & 1.000
			&\textbf{ 0.633} & \textbf{0.512} & \textbf{0.590}
			\\ \hline      
			\hline\multirowcell{2}{Trained on \\ \textbf{CIFAR-100}} &
			\multicolumn{3}{c||}{Plain (TE: \textbf{31.97\%})}  &
			\multicolumn{3}{c||}{CEDA (TE: 32.74\%)} &
			\multicolumn{3}{c|}{ACET (TE: 32.24\%)} \\
			& MMC    & AUROC    & FPR@95    & MMC    & AUROC  & FPR@95    & MMC    & AUROC    &FPR@95 \\ \hline \hline
			CIFAR-100
			& \textbf{0.751} & -- & -- & 0.734 & -- & -- & 0.728 & -- & -- \\ \hline
			SVHN
			& 0.570 & 0.710 & 0.865
			& 0.290 & 0.874 & 0.410
			& \textbf{0.234} & \textbf{0.912} & \textbf{0.345}
			\\ \hline
			CIFAR-10
			& 0.560 & 0.718 & 0.856
			& 0.547 & 0.711 & 0.855
			& \textbf{0.530} & \textbf{0.720} & \textbf{0.860}
			\\ \hline
			LSUN CR
			& 0.592 & 0.690 & 0.887
			& 0.581 & 0.678 & 0.887
			& \textbf{0.554} & \textbf{0.698} & \textbf{0.881}
			\\ \hline
			Imagenet-
			& 0.531 & 0.744 & 0.827
			& 0.504 & 0.749 & 0.808
			& \textbf{0.492} & \textbf{0.752} & \textbf{0.819}
			\\ \hline
			Noise
			& 0.614 & 0.672 & 0.928
			& 0.010 & 1.000 & 0.000
			& \textbf{0.010} & \textbf{1.000} & \textbf{0.000}
			\\ \hline
			Adv. Noise
			& 1.000 & 0.000 & 1.000
			& 0.985 & 0.015 & 0.985
			& \textbf{0.013} & \textbf{0.998} & \textbf{0.003}
			\\ \hline
			Adv. Samples
			& 0.999 & 0.010 & 1.000
			& 0.999 & 0.012 & 1.000
			& \textbf{0.863} & \textbf{0.267} & \textbf{0.975}
			\\ \hline
		\end{tabular}
	\end{center}
	\caption{\label{tab:mainresults} On the four datasets MNIST, SVHN, CIFAR-10, and CIFAR-100, we train three models: Plain, CEDA and ACET. We evaluate them on out-of-distribution samples (other image datasets, noise, adversarial
		noise and adversarial samples built from the test set on which was trained).
		We report test error of all models and show the mean maximum confidence (MMC) on the in- and out-distribution
		samples (lower is better for out-distribution samples), the AUC of the ROC curve (AUROC) for the discrimination between in- and out-distribution based on confidence value (higher is better), and the FPR at 95\% true positive rate for the same problem (lower is better).}
\end{table*}

\section{Experiments}\label{sec:exp}
In the evaluation, we follow \cite{HenGim2017,LiaLiSri2018,LeeEtAl2018} by training on one dataset and evaluating the confidence on other out of distribution datasets and noise images. In contrast to
\cite{LiaLiSri2018,LeeEtAl2018} we neither use a different parameter set for each test dataset \cite{LiaLiSri2018} nor do we use one of the test datasets during training \cite{LeeEtAl2018}. More precisely, we train on MNIST, SVHN, CIFAR-10 and CIFAR-100, where we use the LeNet architecture on MNIST taken from \cite{MadEtAl2018} and a ResNet architecture \cite{he2016deep} for the other datasets. We also use standard data augmentation which includes random crops for all datasets and random mirroring for CIFAR-10 and CIFAR-100. For the
generation of out-of-distribution images from $p_{\textrm{out}}$ we proceed as follows: half of the images are generated by randomly permuting pixels of images from the training set and
half of the images are generated uniformly at random. Then we apply to these images a Gaussian filter with standard deviation $\sigma \in [1.0,2.5]$ as lowpass filter to have more low-frequency structure in the noise.
As the Gaussian filter leads to a contrast reduction we apply afterwards a global rescaling so that the maximal range of the image is again in $[0,1]$. 
\\
\textbf{Training:} We train each model normally (plain), with confidence enhancing data augmentation (CEDA) and with adversarial confidence enhancing training (ACET). It is well known that weight decay alone reduces overconfident predictions. Thus we use weight decay with regularization parameter $5 \cdot 10^{-4}$ for all models leading to a strong baseline (plain). For both CEDA \eqref{eq:CEDA} and ACET \eqref{eq:ACET} we use $\lambda=1$, that means $50\%$ of the samples in each batch are from the original training set and $50\%$ are noise samples as described before. For ACET we use $p=\infty$ and $\epsilon=0.3$ and optimize with PGD \cite{MadEtAl2018} using 40 iterations and stepsize $0.0075$ for all datasets. All models are trained for 100 epochs with ADAM \cite{KinEtAl2014} on MNIST and SGD+momentum for SVHN/CIFAR-10/CIFAR-100. The initial learning rate is $10^{-3}$ for MNIST and $0.1$ for SVHN/CIFAR-10 and it is reduced by a factor of 10 at the $50$th, $75$th and $90$th of the in total $100$ epochs.
The code is available at \url{https://github.com/max-andr/relu_networks_overconfident}. 
\begin{table*}
	\begin{center}
		\begin{tabular}{|l||c|c|c|c||c|c|c|c|}
			\hline
			& \multicolumn{4}{c||}{Plain}  & \multicolumn{4}{c|}{ACET} \\
			& MNIST & SVHN & CIFAR-10 & CIFAR-100    &    MNIST & SVHN & CIFAR-10 & CIFAR-100  \\ 
			\hline \hline
			Median $\alpha$
			& 1.5  & 28.1 & 8.1 & \textbf{9.9}  & $\mathbf{3.0 \cdot 10^{15}}$ & \textbf{49.8}  & \textbf{45.3} & \textbf{9.9}\\
			\hline    
			\% overconfident
			& 98.7\% & 99.9\% & 99.9\% & 99.8\%   & \textbf{0.0\%} & \textbf{50.2\%}  & \textbf{3.4\%} & \textbf{0.0\%}\\
			\hline
		\end{tabular}
	\end{center}
	\caption{\label{tab:alpha_main} 
		\textbf{First row:} We evaluate all trained models on uniform random inputs scaled by a constant $\alpha \geq 1$ (note that the resulting inputs will not constitute valid images anymore, since in most cases they exceed the $[0, 1]^d$ box). We find the minimum $\alpha$ such that the models output 99.9\% confidence on them, and report the median over 10 000 trials. As predicted by Theorem \ref{th:relu} we observe that it is always possible to obtain overconfident predictions just by scaling inputs by some constant $\alpha$, and for plain models this constant is smaller than for ACET. 
		\textbf{Second row:} we show the percentage of overconfident predictions (higher than 95\% confidence) when projecting back the $\alpha$-rescaled uniform noise images back to $[0, 1]^d$. One observes that there are much less overconfident predictions for ACET compared to standard training.
	}
\end{table*}
\begin{figure*}[ht]
\begin{center}
\begin{tabular}{ccc}
Plain & CEDA & ACET\\
\includegraphics[width=0.62\columnwidth]{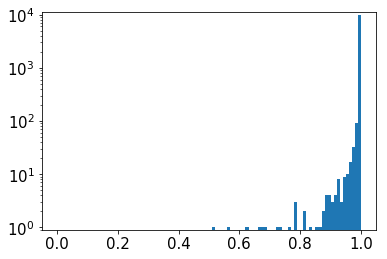}&
\includegraphics[width=0.62\columnwidth]{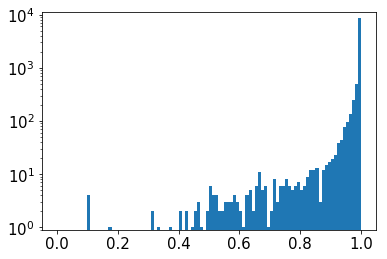}&
\includegraphics[width=0.62\columnwidth]{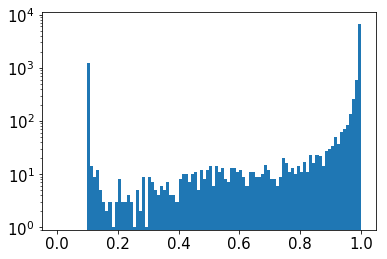}
\end{tabular}
 \end{center}
   \caption{\label{hist:mnist_main}Histogram of confidence values (logarithmic scale) of adversarial samples based on MNIST test points. 
   ACET is the only model where a significant fraction of adversarial samples have very low confidence. Note, however that the ACET model has not been trained on adversarial samples of MNIST, but only on adversarial noise. }
\end{figure*}
\\
\textbf{Evaluation:} We report for each model (plain, CEDA, ACET) the test error and the mean maximal confidence (for each point this is $\max_{k=1,\ldots,K} \frac{e^{f_k(x)}}{\sum_{l=1}^K e^{f_l(x)}}$), denoted as MMC, on the test set. In order to evaluate how well we reduce the confidence on the out-distribution, we use four datasets on CIFAR-10 \cite{cifar10} and SVHN \cite{SVHN} (namely among CIFAR-10, CIFAR-100, SVHN, ImageNet-, which is a subset of ImageNet where we removed classes similar to CIFAR-10,  and the classroom subset of LSUN \cite{LSUN} we use the ones
on which we have \emph{not} trained) and for MNIST we evaluate on EMNIST \cite{CohEtAl2017}, a grayscale version of CIFAR-10 and Fashion MNIST \cite{XiaoEtAl2017}. Additionally, we show the evaluation on noise, adversarial noise and adversarial samples. The noise is generated in the same way as the noise we use for training. For adversarial noise, where we maximize the maximal confidence over all classes
(see $L_{p_{\textrm{out}}}(f,z)$ in \eqref{eq:Lpout}), we use PGD 
with 200 iterations and stepsize $0.0075$ in the $\epsilon$ ball wrt the $\norm{\cdot}_\infty$-norm with $\epsilon=0.3$ (same as in training). Note that for training we use only $40$ iterations, so that the attack at test time is significantly stronger. Finally, we check also the confidence on adversarial samples computed for the test set of the in-distribution dataset  using $80$ iterations of PGD with $\epsilon=0.3$, stepsize $0.0075$ for MNIST and $\epsilon=0.1$, stepsize $0.0025$ for the other datasets. The latter two evaluation modalities are novel compared to
\cite{HenGim2017,LiaLiSri2018,LeeEtAl2018}. The adversarial noise is interesting as it actively searches for images which still yield high confidence in a neighborhood of a noise image and thus is a much more challenging than the pure evaluation on noise. Moreover, it potentially detects an over-adaptation to the noise model used during training in particular in CEDA. The evaluation on adversarial samples is interesting as one can hope that the reduction of the confidence for out-of-distribution images also reduces the confidence of adversarial samples as typically adversarial samples are off the data manifold \cite{StuHeiSch2019} and thus are also out-of-distribution samples (even though their distance to the true distribution is small). Note that our models have never seen adversarial samples during training, they only have
been trained using the adversarial noise. Nevertheless our ACET model can reduce the confidence on adversarial samples.
As evaluation criteria we use the mean maximal confidence, the area under the ROC curve (AUC) where we use the confidence as a threshold for the detection problem (in-distribution vs. out-distribution). Moreover, we report in the same setting the false positive rate (FPR) when the true positive rate (TPR) is fixed to $95\%$. All results can be found in
Table \ref{tab:mainresults}.

\textbf{Main Results:}
In Table \ref{tab:mainresults}, we show the results of plain (normal training), CEDA and ACET. First of all, we observe that there is almost no difference between the test errors of
all three methods. Thus improving the confidence far away from the training data does not impair the generalization performance. We also see that the plain models always produce relatively high
confidence predictions on noise images and completely fail on adversarial noise. CEDA produces low confidence on noise images but mostly fails (except for MNIST) on adversarial noise which was to be expected
as similar findings have been made for the creation of adversarial samples. Only ACET consistently produces low confidence predictions on adversarial noise and has high AUROC. For the 
out-of-distribution datasets, CEDA and ACET improve most of the time the maximal confidence and the AUROC, sometimes with very strong improvements like on MNIST evaluated on FMNIST
or SVHN evaluated on LSUN. However, one observes that it is more difficult to reduce the confidence for related tasks e.g. MNIST evaluated on EMNIST or CIFAR-10 evaluated on LSUN, where
the image structure is more similar.\\
Finally, an interesting outcome is that ACET reduces the confidence on adversarial examples, see Figure \ref{hist:mnist_main} for an illustration for MNIST, and achieves on all datasets improved AUROC values so that one can
detect more adversarial examples via thresholding the confidence compared to the plain and CEDA models. The improved performance of ACET is to some extent unexpected as we just bias the model towards uniform confidence over all classes far away from the training data, but adversarial examples are still close to the original images. In summary, ACET does improve confidence estimates significantly compared to
the plain model but also compared to CEDA, in particular on adversarial noise and adversarial examples. ACET has also a beneficial effect on adversarial examples which is an interesting side effect and shows in our opinion that the models have become more reliable.

\textbf{Far away high confidence predictions:}
Theorem \ref{th:relu} states that ReLU networks always attain high confidence predictions far away from the training data. The two network architectures used in this paper are ReLU networks. It is thus interesting to investigate if the confidence-enhanced training, ACET, makes it harder to reach high confidence than for the plain model. We do the following
experiment: we take uniform random noise images $x$ and then search for the smallest $\alpha$ such that the classifier attains $99.9\%$ confidence on $\alpha x$. This is exactly the construction from Theorem \ref{th:relu} and the result can be found in Table \ref{tab:alpha_main}. \\
We observe that indeed the required upscaling factor $\alpha$ is significantly higher for ACET than for the plain models which implies that
our method also influences the network far away from the training data.
This also shows that even training methods explicitly aiming at counteracting the phenomenon of high confidence predictions far away from the training data, cannot prevent this. We also discuss in the appendix a similar experiment, but with the projection to $[0,1]^d$.



\section{Conclusion}
We have shown in this paper that the problem of arbitrarily high confidence predictions of  ReLU networks far away from the training data cannot be avoided even with modifications
like temperature rescaling \cite{GuoEtAl2017}. It is an inherent problem of the neural network architecture and thus can only be resolved by changing the architecture. On the other hand we have shown that CEDA and in particular ACET are a good way to reach much better confidence estimates for image data.
CEDA and ACET can be directly used for any model with little implementation overhead. For the future it would be desirable to have network architectures which have provably the property that far away from the training data the confidence is uniform over the classes: the network knows when it does not know.

\section*{Acknowledgements}
M.H. and J.B. acknowledge support from the BMBF
through the T{\"u}bingen AI Center (FKZ: 01IS18039A) and by the DFG TRR 248, project number 389792660 
and the DFG Excellence Cluster
``Machine Learning - New Perspectives for Science'', EXC 2064/1, project number 390727645.


{\small
\bibliographystyle{ieee}
\bibliography{Literatur}}

\clearpage

\begin{center}
	\Large\textbf{Appendix}
\end{center}

\setcounter{section}{0}
\renewcommand{\thesection}{\Alph{section}}

\section{Proofs}

\setcounter{section}{3}
\renewcommand{\thesection}{\arabic{section}}

\begin{lemma}\label{le:rays}
	Let $\{Q_i\}_{l=1}^R$ be the set of linear regions associated to the ReLU-classifier $f:\R^d \rightarrow \R^K$. For any $x \in \R^d$
	there exists $\alpha \in \R$ with $\alpha>0$ and $t \in \{1,\ldots,R\}$ such that $\beta x \in Q_t$ for all $\beta \geq \alpha$.
\end{lemma}
\begin{proof}
	Suppose the statement would be false. Then there exist $\{\beta_i\}_{i=1}^\infty$ with $\beta_i\geq 0$, $\beta_i\geq \beta_j$ if $i\leq j$ and 
	$\beta_i \rightarrow \infty$ as $i\rightarrow \infty$ such that for $\gamma \in [\beta_i,\beta_{i+1})$ we have $\gamma x \in Q_{r_i}$ with $r_i \in \{1,\ldots,R\}$ and $r_{i-1}\neq r_i \neq r_{i+1}$. As there are only finitely many regions there exist $i,j \in \N$ with $i<j$ such that $r_i = r_j$, in particular $\beta_i x \in Q_{r_i}$ and $\beta_j x \in Q_{r_i}$.
	However, as the linear regions are convex sets also the line segment $[\beta_i x, \beta_j x] \in Q_{r_i}$. However, that implies $\beta_i=\beta_j$ as neighboring segments are in different regions  which contradicts the assumption. Thus there can only be finitely many $\{\beta_i\}_{i=1}^M$ and the $\{r_i\}_{i=1}^M$ have to be all different, which finishes the proof.
\end{proof}

\begin{theorem}\label{th:relu}
	Let $\R^d=\cup_{l=1}^R Q_l$ and $f(x)=V^l x + a^l$  be the piecewise affine representation of the
	output of a ReLU network on $Q_l$. Suppose that $V^l$ does not contain identical rows for all $l=1,\ldots,R$, then for almost any $x \in \R^d$ and $\epsilon>0$
	there exists an $\alpha>0$ and a class $k \in \{1,\ldots,K\}$ such that for $z=\alpha x$ it holds
	\[ \frac{e^{f_k(z)}}{\sum_{r=1}^K e^{f_r(z)}} \geq 1-\epsilon.\]
	Moreover, $\lim\limits_{\alpha \rightarrow \infty}  \frac{e^{f_k(\alpha x)}}{\sum_{r=1}^K e^{f_r(\alpha x)}} =1$.
\end{theorem}
\begin{proof}
	By Lemma \ref{le:rays} there exists a region $Q_t$ with $t\in \{1,\ldots,R\}$ and $\beta>0$ such that for all $\alpha\geq \beta$ we have  $\alpha x \in Q_t$. 
	Let $f(z)=V^t z + a^t$ be the affine form of the ReLU classifier $f$ on $Q_t$. Let $k^*=\argmax_k \inner{v^t_k,x}$, where $v^t_k$ is the $k$-th row of $V^t$.
	As $V^t$ does not contain identical rows, that is $v^t_l \neq v^t_m$ for $l\neq m$, the maximum is uniquely attained up to a set of measure zero.   
	If the maximum is unique, it holds for sufficiently large $\alpha\geq \beta$
	\begin{align}\label{eq:inequality}
	\inner{v^t_l - v^t_{k^*}, \alpha x} + a^t_l - a^t_{k^*} < 0, \; \forall l \in \{1,\ldots,K\}\backslash \{k^*\}.
	\end{align}
	Thus $\alpha x \in Q_t$ is classified as $k^*$. Moreover,
	\begin{align}
	\frac{e^{f_{k^*}(	\alpha x)}}{\sum_{l=1}^K e^{f_l(\alpha x)}}
	&= \frac{e^{\inner{v^t_{k^*},\alpha x}+a^t_k}}{\sum_{l=1}^K e^{\inner{v^t_l,\alpha x}+a^t_l}}\\
	&=\frac{1}{1 + \sum_{l\neq k^*}^K e^{\inner{v^t_l-v^t_{k^*},\alpha x}+a^t_l-a^t_k}}.
	\end{align}
	By inequality \eqref{eq:inequality} all the terms in the exponential are negative and thus by upscaling $\alpha$, using $\inner{v^t_{k^*},x} > \inner{v^t_l,x}$ for all $l\neq k^*$,
	we can get the exponential term arbitrarily close to $0$. In particular, 
	\[ \lim\limits_{\alpha \rightarrow \infty} \frac{1}{1 + \sum_{l\neq k}^K e^{\inner{v^t_l-v^t_{k^*},\alpha x}+a^t_l-a^t_k}} =1.\]
\end{proof}

\begin{theorem}\label{th:RBF}
	Let $f_k(x)=\sum_{l=1}^N \alpha_{kl} e^{-\gamma \norm{x-x_l}^2_2}$, $k=1,\ldots,K$ be an RBF-network trained with cross-entropy loss on the training data $(x_i,y_i)_{i=1}^N$. We define
	$r_{\min}=\minop_{l=1,\ldots,N} \norm{x-x_l}_2$ and $\alpha=\maxop_{r,k} \sum_{l=1}^N |\alpha_{rl}-\alpha_{kl}|$. If $\epsilon>0$ and 
	\[ r^2_{\min} \geq \frac{1}{\gamma}\log\Big(\frac{\alpha}{\log(1+K\epsilon)}\Big),\]
	then for all $k=1,\ldots,K$,
	\[ \frac{1}{K}-\epsilon \; \leq\; \frac{e^{f_k(x)}}{\sum_{r=1}^K e^{f_r(x)}} \; \leq \; \frac{1}{K}+\epsilon.\] 
\end{theorem}
\begin{proof}
	It holds $\frac{e^{f_k(x)}}{\sum_{r=1}^K e^{f_r(x)}}=\frac{1}{\sum_{r=1}^K e^{f_r(x)-f_k(x)}}$. With
	\begin{align}
	|f_r(x)-f_k(x)| &=\big| \sum_l (\alpha_{rl}-\alpha_{kl})e^{-\gamma \norm{x-x_l}^2_2}\big|\\
	&\leq e^{-\gamma r_{\min}^2} \sum_l |\alpha_{rl}-\alpha_{kl}|\\
	&\leq e^{-\gamma r_{\min}^2} \alpha \leq \log(1+K\epsilon),
	\end{align}
	where the last inequality follows by the condition on $r_{\min}$. We get
	\begin{align}
	\frac{1}{\sum_{r=1}^K e^{f_r(x)-f_k(x)}} &\geq \frac{1}{\sum_{r=1}^K e^{|f_r(x)-f_k(x)|}}\\
	&\geq \frac{1}{K e^{\alpha e^{-\gamma r_{\min}^2}}}\\
	&\geq \frac{1}{K}\frac{1}{1+K\epsilon} \geq \frac{1}{K}-\epsilon,
	\end{align}
	where we have used in the third inequality the condition on $r_{\min}^2$ and in the last step we use
	$1 \geq (1-K\epsilon)(1+K\epsilon)=1-K^2\epsilon^2$. Similarly, we get
	\begin{align*}
	\frac{1}{\sum_{r=1}^K e^{f_r(x)-f_k(x)}} &\leq \frac{1}{\sum_{r=1}^K e^{-|f_r(x)-f_k(x)|}}\\
	&\leq \frac{1}{K e^{-\alpha e^{-\gamma r_{\min}^2}}}\\
	&\leq \frac{1}{K} (1+K\epsilon) \leq \frac{1}{K}+\epsilon.
	\end{align*}
	This finishes the proof.
\end{proof}

\setcounter{section}{1}
\renewcommand{\thesection}{\Alph{section}}

\section{Additional $\alpha$-scaling experiments}
We also do a similar $\alpha$-scaling experiment, but with the projection to the image domain ($[0, 1]^d$ box), and report the percentage of overconfident predictions (higher than 95\% confidence) in Table \ref{tab:alpha_main}, second row. We observe that such a technique can lead to overconfident predictions even in the image domain for the plain models. At the same time, on all datasets, the ACET models have a significantly smaller fraction of overconfident examples compared to the plain models.


\section{The effect of Adversarial Confidence Enhanced Training}

In this section we compare predictions of the plain model trained on MNIST (Figure \ref{fig:mnist_plain_most_least_confident}) and the model trained with ACET (Figure \ref{fig:mnist_acet_most_least_confident}). We analyze the images that receive the lowest maximum confidence on the original dataset (MNIST), and the highest maximum confidence on the two datasets that were used for evaluation (EMNIST, grayCIFAR-10).

\textbf{Evaluated on MNIST}: We observe that for both models the lowest maximum confidence corresponds to hard input images that are either discontinous, rotated or simply ambiguous.

\textbf{Evaluated on EMNIST}: Note that some handwritten letters from EMNIST, e.g. 'o' and 'i' may look exactly the same as digits '0' and '1'. Therefore, one should not expect that an ideal model assigns uniform confidences to all EMNIST images.  For Figure \ref{fig:mnist_plain_most_least_confident} and Figure \ref{fig:mnist_acet_most_least_confident} we consider predictions on letters that in general do not look exactly like digits ('a', 'b', 'c', 'd'). We observe that the images with the highest maximum confidence correspond to the handwritten letters that \textit{resemble} digits, so the predictions of both models are justified.

\textbf{Evaluated on Grayscale CIFAR-10}: This dataset consists of the images that are clearly distinct from digits. Thus, one can expect uniform confidences on such images, which is achieved by the ACET model (Table 1), but not with the plain model. The mean maximum confidence of the ACET model is close to 10\%, with several individual images that are scored with up to 40.41\% confidence. Note, that this is much better than for the plain model, which assigns up to 99.60\% confidence for the images that have nothing to do with digits. This result is particularly interesting, since the ACET model has not been trained on grayCIFAR-10 examples, and yet it shows much better confidence calibration for out-of-distribution samples.

\def\mnistmnistfolder{}
\def\mnistmnist{
	mnistplainmnist_top-1_-_Pred_1_true_2_with_p=37_58/1/37.58,
	mnistplainmnist_top-2_-_Pred_1_true_6_with_p=39_72/1/39.72,
	mnistplainmnist_top-3_-_Pred_7_true_1_with_p=40_49/7/40.49,
	mnistplainmnist_top-4_-_Pred_7_true_7_with_p=40_54/7/40.54,
	mnistplainmnist_top-5_-_Pred_5_true_6_with_p=43_31/5/43.31,
	mnistplainmnist_top-6_-_Pred_9_true_9_with_p=45_73/9/45.73,
	mnistplainmnist_top-7_-_Pred_1_true_7_with_p=47_86/1/47.86}

\def\mnistemnistfolder{}
\def\mnistemnist{
	mnistplainemnist_top-1_-_Pred_0_true_4_with_p=100_00/0/100.0,
	mnistplainemnist_top-2_-_Pred_0_true_4_with_p=100_00/0/100.0,
	mnistplainemnist_top-3_-_Pred_6_true_2_with_p=100_00/6/100.0,
	mnistplainemnist_top-4_-_Pred_6_true_2_with_p=100_00/6/100.0,
	mnistplainemnist_top-5_-_Pred_0_true_4_with_p=100_00/0/100.0,
	mnistplainemnist_top-6_-_Pred_0_true_4_with_p=100_00/0/100.0,
	mnistplainemnist_top-7_-_Pred_0_true_4_with_p=100_00/0/100.0}

\def\mnistcifarfolder{}
\def\mnistcifar{
	mnistplaincifar_top-1_-_Pred_2_true_7_with_p=99_60/2/99.60,
	mnistplaincifar_top-2_-_Pred_2_true_6_with_p=99_13/2/99.13,
	mnistplaincifar_top-3_-_Pred_7_true_5_with_p=98_99/7/98.99,
	mnistplaincifar_top-4_-_Pred_6_true_3_with_p=98_83/6/98.83,
	mnistplaincifar_top-5_-_Pred_2_true_0_with_p=98_76/2/98.76,
	mnistplaincifar_top-6_-_Pred_7_true_9_with_p=98_65/7/98.65,
	mnistplaincifar_top-7_-_Pred_6_true_3_with_p=98_48/6/98.48}

\newcounter{leastmostcounter}
\begin{center}
	\begin{figure*}[h]
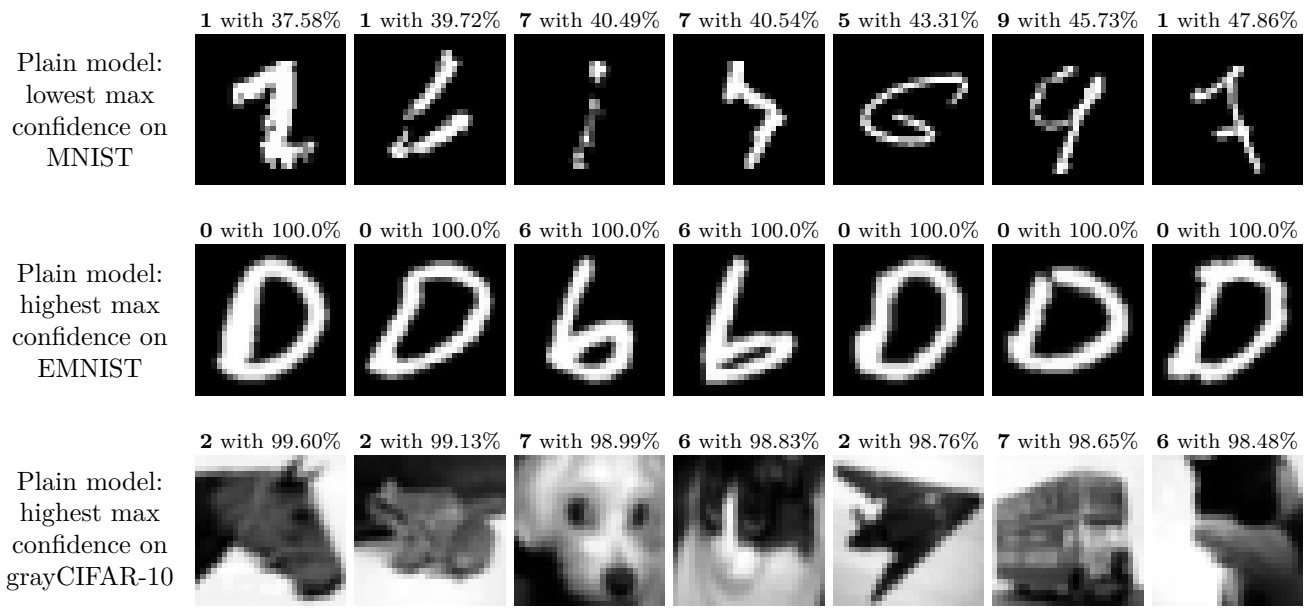

		\begin{tikzpicture}	
		\node[text width=2.4cm, align=center] at (-2.4, 0) {Plain model: \\ lowest max \\ confidence on \\ MNIST};
		\node[text width=2.4cm, align=center] at (-2.4, -2.8) {Plain model: \\ highest max \\ confidence on \\ EMNIST};
		\node[text width=2.4cm, align=center] at (-2.4, -2.8*2) {Plain model: \\ highest max \\ confidence on \\ grayCIFAR-10};
		
		\footnotesize
		\foreach \fname / \class / \conf  in \mnistmnist {    
			\node[inner sep=0pt] (s0) at (2.12*\value{leastmostcounter},1.2)    {\textbf{\class} with \conf\%};
			\node[inner sep=0pt] (s0) at (2.12*\value{leastmostcounter},0)      {\includegraphics[width=2.0cm]{\fname.png}};
			\stepcounter{leastmostcounter}
		}
		
		\setcounter{leastmostcounter}{0}
		\foreach   \fname / \class / \conf in \mnistemnist {      
			\node[inner sep=0pt] (s1) at (2.12*\value{leastmostcounter},-1.6)    {\textbf{\class} with \conf\%};  
			\node[inner sep=0pt] (s1) at (2.12*\value{leastmostcounter},-2.8)      {\includegraphics[width=2.0cm]{\fname.png}};
			\stepcounter{leastmostcounter}
		}
		
		\setcounter{leastmostcounter}{0}
		\foreach   \fname / \class / \conf in \mnistcifar {        
			\node[inner sep=0pt] (s2) at (2.12*\value{leastmostcounter},-4.4)    {\textbf{\class} with \conf\%};
			\node[inner sep=0pt] (s2) at (2.12*\value{leastmostcounter},-2.8*2)      {\includegraphics[width=2.0cm]{\fname.png}};
			\stepcounter{leastmostcounter}
		}
		
		\end{tikzpicture}
		\caption{\label{fig:mnist_plain_most_least_confident}
			Top Row: predictions of the plain MNIST model with the lowest maximum confidence.
			Middle Row: predictions of the plain MNIST model on letters 'a', 'b', 'c', 'd' of EMNIST with the highest maximum confidence.
			Bottom Row: predictions of the plain MNIST model on the grayscale version of CIFAR-10 with the highest maximum confidence.
			Note that although the predictions on EMNIST are mostly justified, the predictions on CIFAR-10 are overconfident on the images that have no resemblance to digits.
		}
	\end{figure*}
\end{center}

\def\mnistmnistfolder{}
\def\mnistmnist{
	mnistacetmnist_top-1_-_Pred_1_true_6_with_p=26_80/1/26.80,
	mnistacetmnist_top-2_-_Pred_1_true_3_with_p=35_73/1/35.73,
	mnistacetmnist_top-3_-_Pred_3_true_7_with_p=36_21/3/36.21,
	mnistacetmnist_top-4_-_Pred_7_true_1_with_p=36_83/7/36.83,
	mnistacetmnist_top-5_-_Pred_3_true_9_with_p=38_00/3/38.00,
	mnistacetmnist_top-6_-_Pred_3_true_9_with_p=38_91/3/38.91,
	mnistacetmnist_top-7_-_Pred_2_true_2_with_p=39_86/2/39.86}

\def\mnistemnistfolder{}
\def\mnistemnist{
	mnistacetemnist_top-1_-_Pred_2_true_1_with_p=100_00/2/100.0,
	mnistacetemnist_top-2_-_Pred_6_true_2_with_p=100_00/6/100.0,
	mnistacetemnist_top-3_-_Pred_2_true_1_with_p=99_99/2/99.99,
	mnistacetemnist_top-4_-_Pred_6_true_2_with_p=99_99/6/99.99,
	mnistacetemnist_top-5_-_Pred_6_true_2_with_p=99_99/6/99.99,
	mnistacetemnist_top-6_-_Pred_6_true_2_with_p=99_99/6/99.99,
	mnistacetemnist_top-7_-_Pred_0_true_4_with_p=99_99/0/99.99}

\def\mnistcifarfolder{}
\def\mnistcifar{
	mnistacetcifar_top-1_-_Pred_0_true_1_with_p=40_41/0/40.41,
	mnistacetcifar_top-2_-_Pred_4_true_8_with_p=38_24/4/38.24,
	mnistacetcifar_top-3_-_Pred_0_true_6_with_p=36_13/0/36.13,
	mnistacetcifar_top-4_-_Pred_0_true_2_with_p=34_91/0/34.91,
	mnistacetcifar_top-5_-_Pred_0_true_3_with_p=34_37/0/34.37,
	mnistacetcifar_top-6_-_Pred_0_true_7_with_p=33_58/0/33.58,
	mnistacetcifar_top-7_-_Pred_7_true_2_with_p=32_36/7/32.36}

\setcounter{leastmostcounter}{0}
\begin{center}
	\begin{figure*}[h]
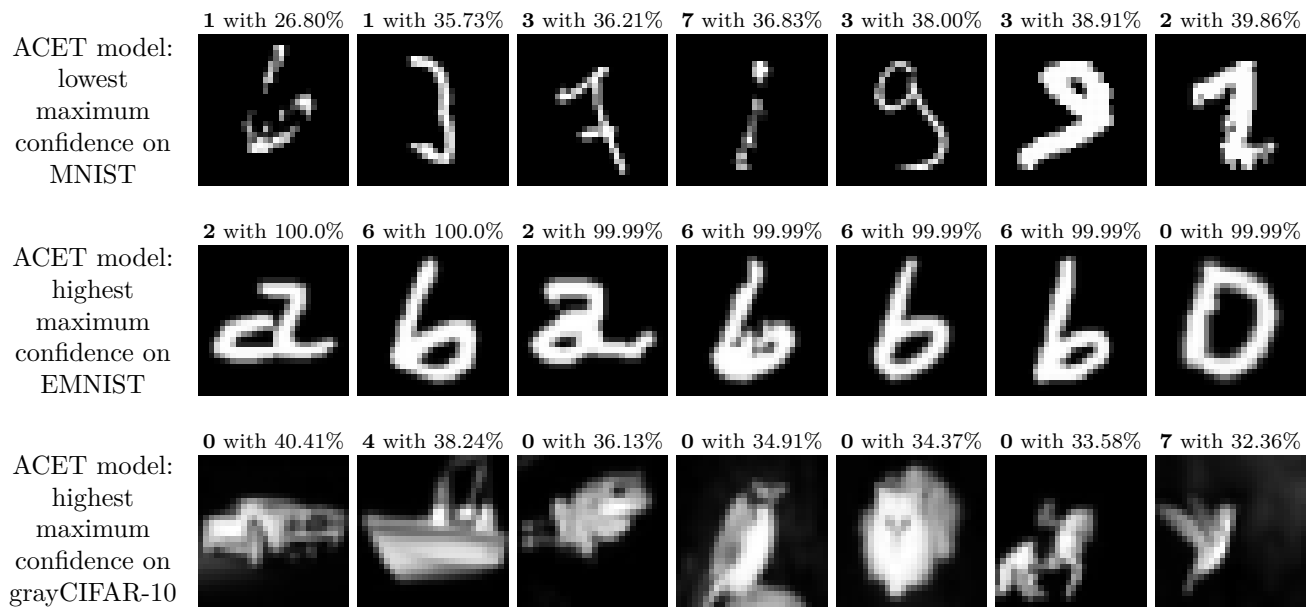

		\begin{tikzpicture}	
		\node[text width=2.4cm, align=center] at (-2.4, 0) {ACET model: \\ lowest maximum \\ confidence on \\ MNIST};
		\node[text width=2.4cm, align=center] at (-2.4, -2.8) {ACET model: \\ highest maximum \\ confidence on \\ EMNIST};
		\node[text width=2.4cm, align=center] at (-2.4, -2.8*2) {ACET model: \\ highest maximum \\ confidence on \\ grayCIFAR-10};
		
		\footnotesize
		\foreach \fname / \class / \conf  in \mnistmnist {    
			\node[inner sep=0pt] (s0) at (2.12*\value{leastmostcounter},1.2)    {\textbf{\class} with \conf\%};
			\node[inner sep=0pt] (s0) at (2.12*\value{leastmostcounter},0)      {\includegraphics[width=2.0cm]{\fname.png}};
			\stepcounter{leastmostcounter}
		}
		
		\setcounter{leastmostcounter}{0}
		\foreach   \fname / \class / \conf in \mnistemnist {      
			\node[inner sep=0pt] (s1) at (2.12*\value{leastmostcounter},-1.6)    {\textbf{\class} with \conf\%};  
			\node[inner sep=0pt] (s1) at (2.12*\value{leastmostcounter},-2.8)      {\includegraphics[width=2.0cm]{\fname.png}};
			\stepcounter{leastmostcounter}
		}
		
		\setcounter{leastmostcounter}{0}
		\foreach   \fname / \class / \conf in \mnistcifar {        
			\node[inner sep=0pt] (s2) at (2.12*\value{leastmostcounter},-4.4)    {\textbf{\class} with \conf\%};
			\node[inner sep=0pt] (s2) at (2.12*\value{leastmostcounter},-2.8*2)      {\includegraphics[width=2.0cm]{\fname.png}};
			\stepcounter{leastmostcounter}
		}
		
		\end{tikzpicture}
		\caption{\label{fig:mnist_acet_most_least_confident}
			Top Row: predictions of the ACET MNIST model with the lowest maximum confidence.
			Middle Row: predictions of the ACET MNIST model on letters 'a', 'b', 'c', 'd' of EMNIST with the highest maximum confidence.
			Bottom Row: predictions of the ACET MNIST model on the grayscale version of CIFAR-10 with the highest maximum confidence.
			Note that for the ACET model the predictions on both EMNIST and grayCIFAR-10 are now justified.
		}
	\end{figure*}
\end{center}

\section{ROC curves}
We show the ROC curves for the binary classification task of separating \textit{True} (in-distribution) images from \textit{False} (out-distribution) images. These correspond to the AUROC values (area under the ROC curve) reported in Table 1 in the main paper. As stated in the paper the separation of in-distribution from out-distribution is done by thresholding the maximal
confidence value over all classes taken from the original multi-class problem.
Note that the ROC curve shows on the vertical axis the True Positive Rate (TPR), and the horizontal axis is the False Positive Rate (FPR).
Thus the FPR@95\%TPR value can be directly read off from the ROC curve as the FPR value achieved for 0.95 TPR. 
Note that a value of $1$ of AUROC corresponds to a perfect classifier. A value below 0.5 means that the
ordering is reversed: out-distribution images achieve on average higher confidence than the in-distribution images. The worst case is an AUROC of zero, in which case all
out-distribution images achieve a higher confidence value than the in-distribution images.

\subsection{ROC curves for the models trained on MNIST}
In the ROC curves for the plain, CEDA and ACET models for MNIST that are presented in Figure \ref{roc:mnist}, the different grades of improvements for the six evaluation datasets can be observed. For noise, the curve of the plain model is already quite close to the upper left corner (which means high AUROC), while for the models trained with CEDA and ACET, it actually reaches that corner, which is the ideal case. For adversarial noise, the plain model is worse than a random classifier, which manifests itself in the fact that the ROC curve runs below the diagonal. While CEDA is better, ACET achieves a very good result here as well.

\subsection{ROC curves for the models trained on SVHN}
CEDA and ACET significantly outperform plain training in all metrics. While CEDA and ACET perform similar on 
CIFAR-10, LSUN and noise, ACET outperforms CEDA clearly on adversarial noise and adversarial samples.

\subsection{ROC curves for the models trained on \mbox{CIFAR-10}}
The ROC curves for CIFAR10 show that this dataset is harder than MNIST or SVHN. However, CEDA and ACET improve significantly on SVHN.
For LSUN even plain training is slightly better (only time for all three datasets). However, on noise and adversarial noise ACET outperforms 
all other methods.

\subsection{ROC curves for the models trained on \mbox{CIFAR-100}}
Qualitatively, on CIFAR-100, we observe the same results as for CIFAR-10. Note that the use of the confidences to distinguish between in- and out-distribution examples generally works worse here. This might be attributed to the fact that CIFAR-100 has considerably more classes, and a higher test error. Therefore, the in- and out-distribution confidences are more likely to overlap.

\begin{center}
	\begin{figure*}[h]
		\begin{tabular}{@{\hskip-.5mm}c@{\hskip-.5mm}@{\hskip-.5mm}c@{\hskip-.5mm}@{\hskip-.5mm}c@{\hskip-.5mm}}
			\textbf{Plain} & \textbf{CEDA} & \textbf{ACET} \\
			\raisebox{-.5\height}{\includegraphics[width=0.72\columnwidth]{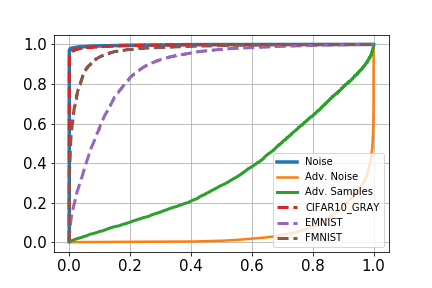}}&
			\raisebox{-.5\height}{\includegraphics[width=0.72\columnwidth]{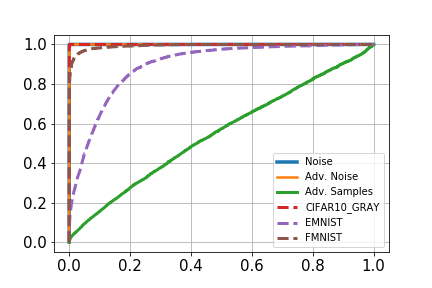}}&
			\raisebox{-.5\height}{\includegraphics[width=0.72\columnwidth]{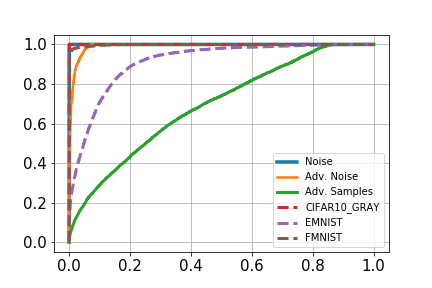}}	
		\end{tabular}
		\begin{center}
			\caption{\label{roc:mnist} ROC curves of the MNIST models on the evaluation datasets.
			}
		\end{center}
	\end{figure*}
\end{center}
\begin{center}
	\begin{figure*}[h]
		\begin{tabular}{@{\hskip-.5mm}c@{\hskip-.5mm}@{\hskip-.5mm}c@{\hskip-.5mm}@{\hskip-.5mm}c@{\hskip-.5mm}}
			\textbf{Plain} & \textbf{CEDA} & \textbf{ACET} \\
			\raisebox{-.5\height}{\includegraphics[width=0.72\columnwidth]{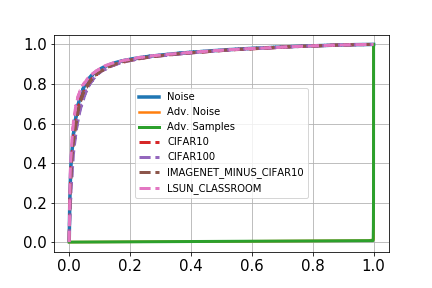}}&
			\raisebox{-.5\height}{\includegraphics[width=0.72\columnwidth]{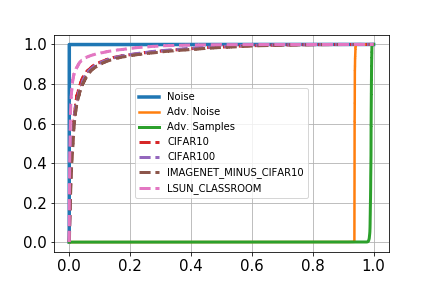}}&
			\raisebox{-.5\height}{\includegraphics[width=0.72\columnwidth]{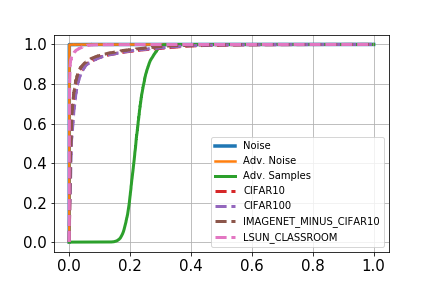}}	
		\end{tabular}
		\begin{center}
			\caption{\label{roc:svhn} ROC curves of the SVHN models on the evaluation datasets.
			}
		\end{center}
	\end{figure*}
\end{center}
\begin{center}
	\begin{figure*}[h]
		\begin{tabular}{@{\hskip-.5mm}c@{\hskip-.5mm}@{\hskip-.5mm}c@{\hskip-.5mm}@{\hskip-.5mm}c@{\hskip-.5mm}}
			\textbf{Plain} & \textbf{CEDA} & \textbf{ACET} \\
			\raisebox{-.5\height}{\includegraphics[width=0.72\columnwidth]{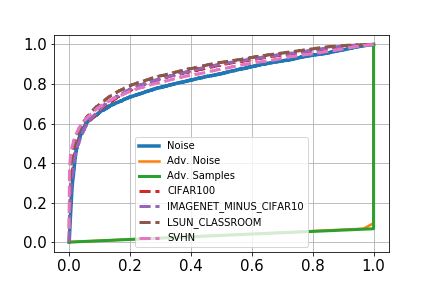}}&
			\raisebox{-.5\height}{\includegraphics[width=0.72\columnwidth]{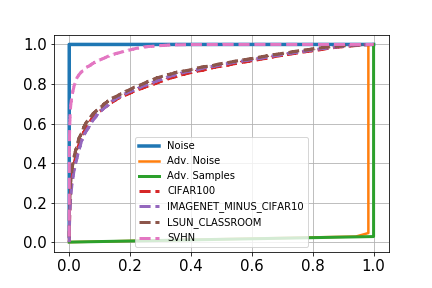}}&
			\raisebox{-.5\height}{\includegraphics[width=0.72\columnwidth]{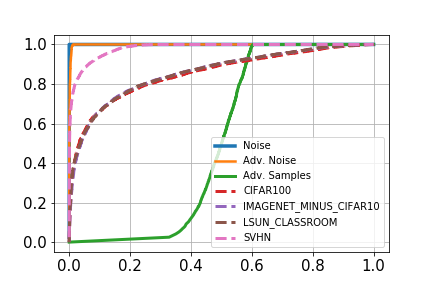}}	
		\end{tabular}
		\begin{center}
			\caption{\label{roc:cifar} ROC curves of the CIFAR-10 models on the evaluation datasets.
			}
		\end{center}
	\end{figure*}
\end{center}
\begin{center}
	\begin{figure*}[h]
		\begin{tabular}{@{\hskip-.5mm}c@{\hskip-.5mm}@{\hskip-.5mm}c@{\hskip-.5mm}@{\hskip-.5mm}c@{\hskip-.5mm}}
			\textbf{Plain} & \textbf{CEDA} & \textbf{ACET} \\
			\raisebox{-.5\height}{\includegraphics[width=0.72\columnwidth]{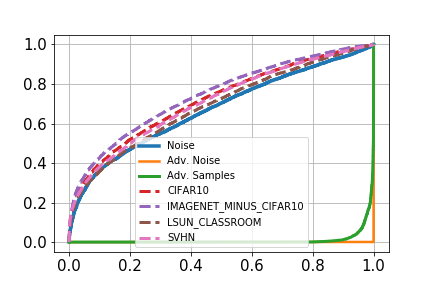}}&
			\raisebox{-.5\height}{\includegraphics[width=0.72\columnwidth]{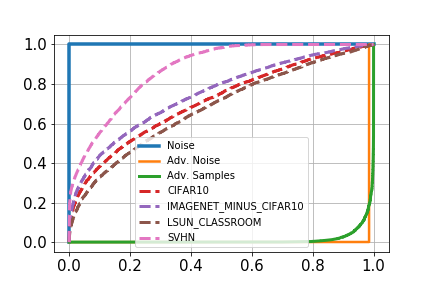}}&
			\raisebox{-.5\height}{\includegraphics[width=0.72\columnwidth]{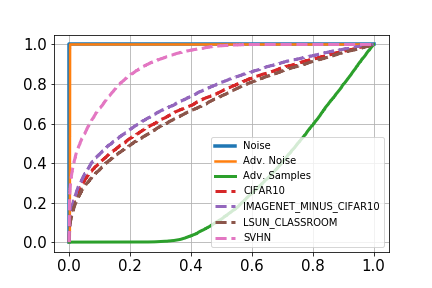}}	
		\end{tabular}
		\begin{center}
			\caption{\label{roc:c100} ROC curves of the CIFAR-100 models on the evaluation datasets.
			}
		\end{center}
	\end{figure*}
\end{center}

\section{Histograms of confidence values}
As the AUROC or the FPR@95\%TPR just tell us how well the confidence values of in-distribution and out-distribution are ordered, we also report
the histograms of achieved confidence values on the original dataset (in-distribution) on which it was trained and the different evaluation datasets.
The histograms show how many times the maximum confidence for test images have certain values between minimal possible $0.1$ ($0.01$ for CIFAR-100) and maximal possible $1.0$. They give a more detailed picture than the single numbers for mean maximum confidence, area under ROC and FPR@95\% TPR.

\subsection{Histograms of confidence values for models trained on MNIST}
As visible in the top row of Figure \ref{hist:mnist-app}, the confidence values for clean MNIST test images don't change significantly for CEDA and ACET.
For FMNIST, gray CIFAR-10 and Noise inputs, the maximum confidences of CEDA are generally shifted to lower values, and those of ACET even more so. For EMNIST, the same effect is observable, though much weaker due to the similarity of characters and digits. For adversarial noise, both CEDA and ACET are very successful in lowering the confidences, with most predictions around 10\% confidence.
As discussed in the main paper, CEDA is not very beneficial for adversarial images, while ACET slightly lowers its confidence to an average value of 85.4\% here.

\subsection{Histograms of confidence values for models trained on SVHN}
Figure \ref{hist:svhn} shows that both CEDA and ACET assign lower confidences to the out-of-distribution samples from SVHN house numbers and LSUN classroom examples. CEDA and ACET, as expected, also signficantly improve on noise samples. While a large fraction of adversarial samples/noise still achieve high confidence values, our ACET trained model is the only one that lowers the confidences for adversarial noise and adversarial samples significantly.

\subsection{Histograms of confidence values for models trained on CIFAR-10}
In Figure \ref{hist:cifar10}, CEDA and ACET lower significantly the confidence on noise, and ACET shows an improvement for adversarial noise, which fools the plain and CEDA models completely. For CIFAR-10, plain and CEDA models yield very high confidence values on adversarial images, while for ACET model the confidence is reduced.
Additionally, on SVHN, we observe a shift towards lower confidence for CEDA and ACET compared to the plain model.

\subsection{Histograms of confidence values for models trained on CIFAR-100}
In Figure \ref{hist:cifar100}, we see similar results to the other datasets. It is noticable in the histograms that for adversarial noise, the deployed attack either achieves 100\% confidence or no improvement at all. For CEDA, the attack succeeds in most cases, and for ACET only rarely.

\clearpage

\begin{figure*}[ht]
	\begin{center}
		\begin{tabular}[t]{cccc}
			\textbf{Dataset} & \textbf{Plain} & \textbf{CEDA} & \textbf{ACET}\\
			\begin{tabular}{c}
				\textbf{MNIST}
			\end{tabular} &
			\raisebox{-.5\height}{\includegraphics[width=0.5\columnwidth]{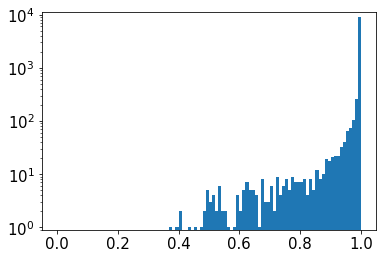}}&
			\raisebox{-.5\height}{\includegraphics[width=0.5\columnwidth]{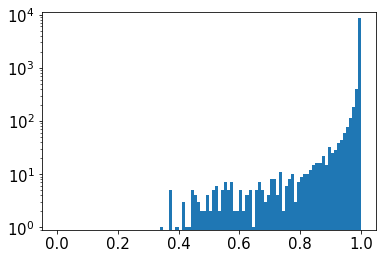}}&
			\raisebox{-.5\height}{\includegraphics[width=0.5\columnwidth]{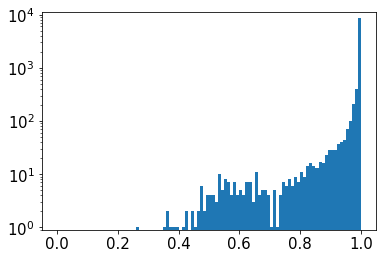}}\\
			\begin{tabular}{c}
				FMNIST 
			\end{tabular} &
			\raisebox{-.5\height}{\includegraphics[width=0.5\columnwidth]{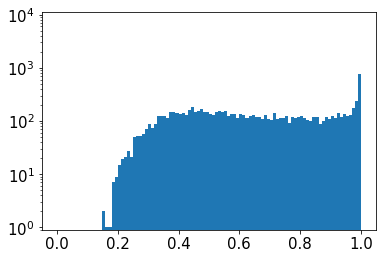}}&
			\raisebox{-.5\height}{\includegraphics[width=0.5\columnwidth]{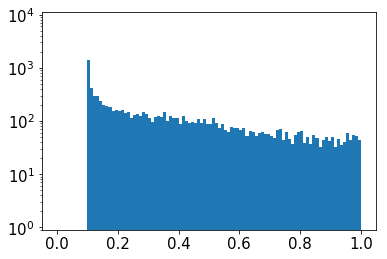}}&
			\raisebox{-.5\height}{\includegraphics[width=0.5\columnwidth]{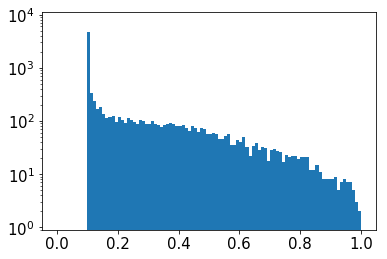}}\\
			\begin{tabular}{c}
				EMNIST 
			\end{tabular} &
			\raisebox{-.5\height}{\includegraphics[width=0.5\columnwidth]{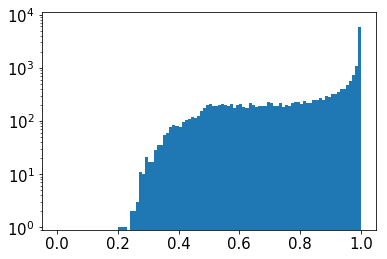}}&
			\raisebox{-.5\height}{\includegraphics[width=0.5\columnwidth]{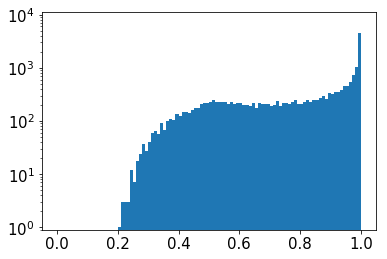}}&
			\raisebox{-.5\height}{\includegraphics[width=0.5\columnwidth]{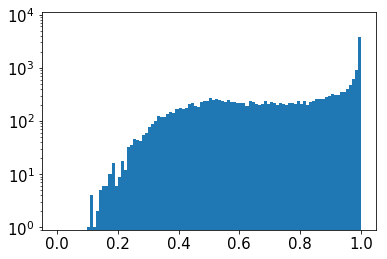}}\\
			\begin{tabular}{c}
				Gray \\ CIFAR-10
			\end{tabular} &
			\raisebox{-.5\height}{\includegraphics[width=0.5\columnwidth]{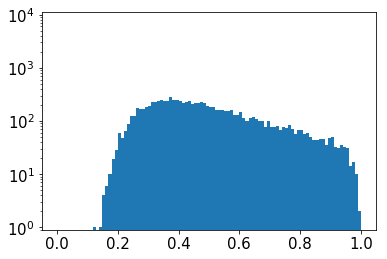}}&
			\raisebox{-.5\height}{\includegraphics[width=0.5\columnwidth]{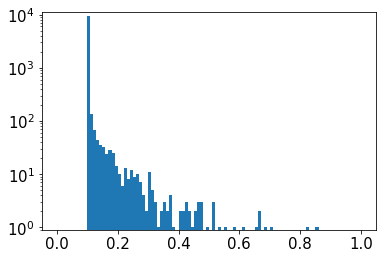}}&
			\raisebox{-.5\height}{\includegraphics[width=0.5\columnwidth]{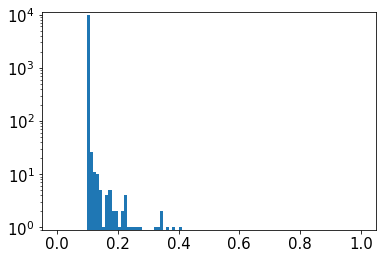}}\\
			\begin{tabular}{c}
				Noise
			\end{tabular} &
			\raisebox{-.5\height}{\includegraphics[width=0.5\columnwidth]{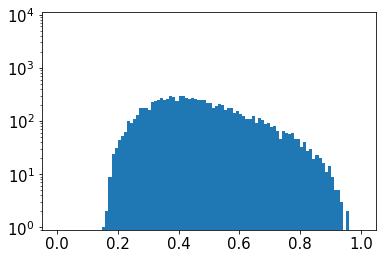}}&
			\raisebox{-.5\height}{\includegraphics[width=0.5\columnwidth]{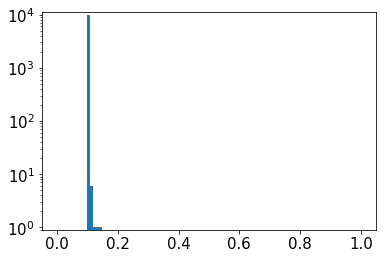}}&
			\raisebox{-.5\height}{\includegraphics[width=0.5\columnwidth]{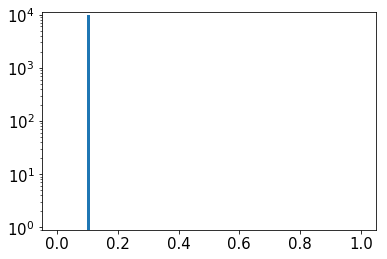}}\\
			\begin{tabular}{c}
				Adversarial \\ Noise
			\end{tabular} &
			\raisebox{-.5\height}{\includegraphics[width=0.5\columnwidth]{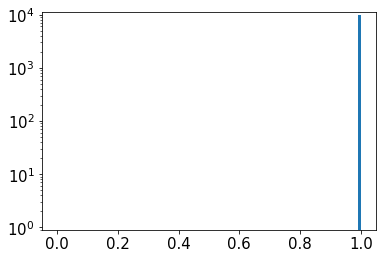}}&
			\raisebox{-.5\height}{\includegraphics[width=0.5\columnwidth]{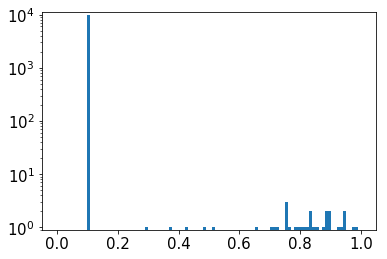}}&
			\raisebox{-.5\height}{\includegraphics[width=0.5\columnwidth]{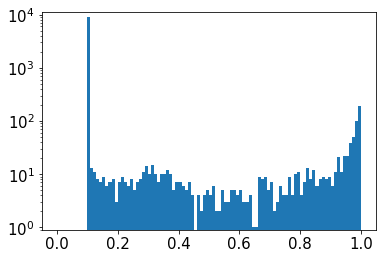}}\\
			\begin{tabular}{c}
				Adversarial \\ Samples
			\end{tabular} &
			\raisebox{-.5\height}{\includegraphics[width=0.5\columnwidth]{eval_graphs_mnistp_conf_adv.png}}&
			\raisebox{-.5\height}{\includegraphics[width=0.5\columnwidth]{eval_graphs_mnistc_conf_adv.png}}&
			\raisebox{-.5\height}{\includegraphics[width=0.5\columnwidth]{eval_graphs_mnista_conf_adv.png}}	
		\end{tabular}
	\end{center}
	\begin{center}
		\caption{\label{hist:mnist-app}Histograms (logarithmic scale) of maximum confidence values of the three compared models for \textbf{MNIST} on various evaluation datasets. 
		}
	\end{center}
\end{figure*}

\pagenumbering{gobble}

\begin{figure*}[ht]
	\begin{center}
		\begin{tabular}[t]{cccc}
			\textbf{Dataset} & \textbf{Plain} & \textbf{CEDA} & \textbf{ACET}\\
			\begin{tabular}{c}
				\textbf{SVHN}
			\end{tabular} &
			\raisebox{-.5\height}{\includegraphics[width=0.5\columnwidth]{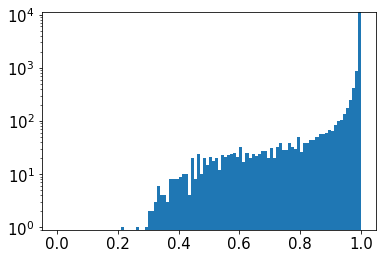}}&
			\raisebox{-.5\height}{\includegraphics[width=0.5\columnwidth]{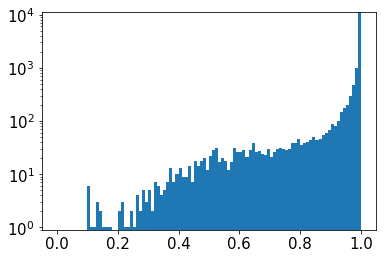}}&
			\raisebox{-.5\height}{\includegraphics[width=0.5\columnwidth]{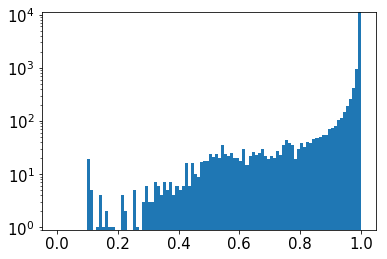}}\\
			\begin{tabular}{c}
				CIFAR-10 
			\end{tabular} &
			\raisebox{-.5\height}{\includegraphics[width=0.5\columnwidth]{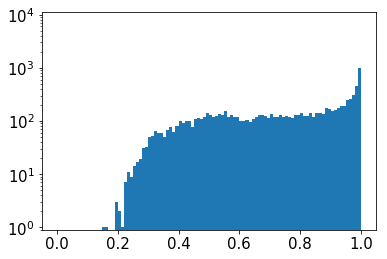}}&
			\raisebox{-.5\height}{\includegraphics[width=0.5\columnwidth]{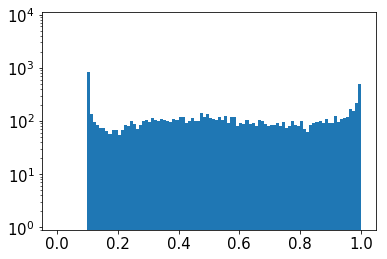}}&
			\raisebox{-.5\height}{\includegraphics[width=0.5\columnwidth]{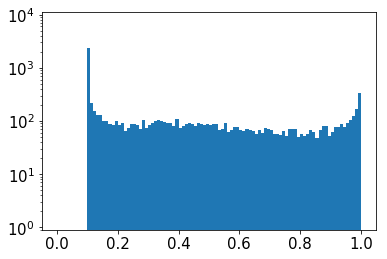}}\\
			\begin{tabular}{c}
				CIFAR-100 
			\end{tabular} &
			\raisebox{-.5\height}{\includegraphics[width=0.5\columnwidth]{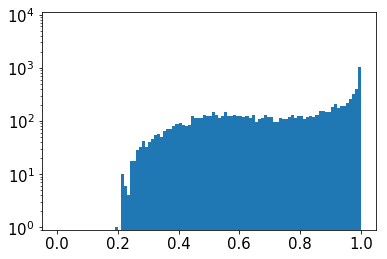}}&
			\raisebox{-.5\height}{\includegraphics[width=0.5\columnwidth]{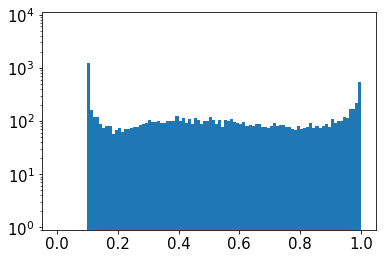}}&
			\raisebox{-.5\height}{\includegraphics[width=0.5\columnwidth]{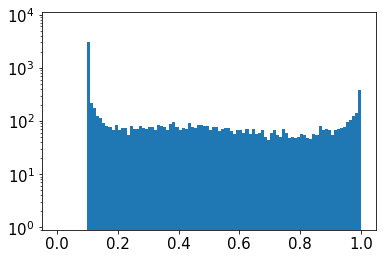}}\\
			\begin{tabular}{c}
				LSUN \\ Classroom
			\end{tabular} &
			\raisebox{-.5\height}{\includegraphics[width=0.5\columnwidth]{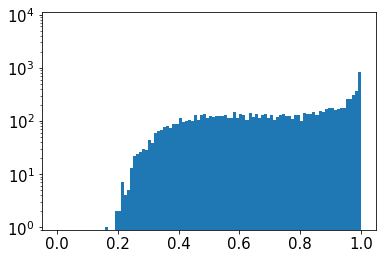}}&
			\raisebox{-.5\height}{\includegraphics[width=0.5\columnwidth]{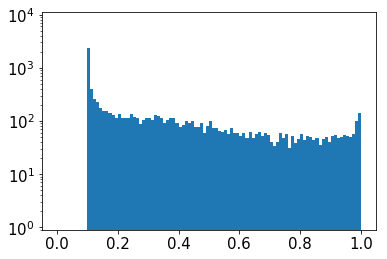}}&
			\raisebox{-.5\height}{\includegraphics[width=0.5\columnwidth]{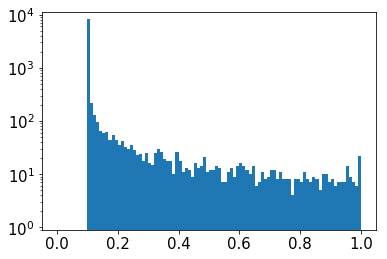}}\\
			\begin{tabular}{c}
				Imagenet \\ minus C10
			\end{tabular} &
			\raisebox{-.5\height}{\includegraphics[width=0.5\columnwidth]{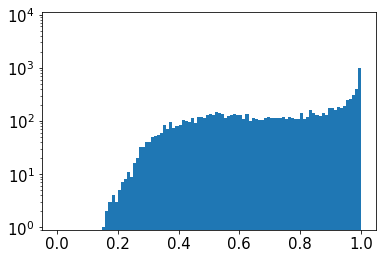}}&
			\raisebox{-.5\height}{\includegraphics[width=0.5\columnwidth]{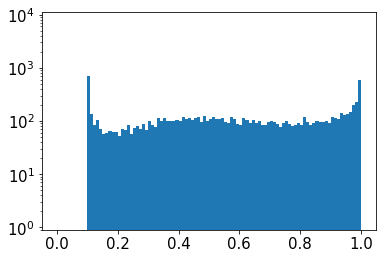}}&
			\raisebox{-.5\height}{\includegraphics[width=0.5\columnwidth]{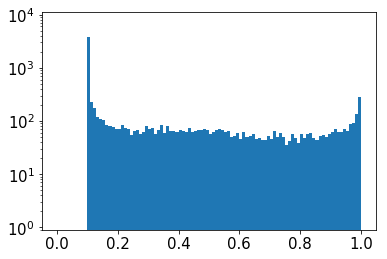}}\\
			\begin{tabular}{c}
				Noise
			\end{tabular} &
			\raisebox{-.5\height}{\includegraphics[width=0.5\columnwidth]{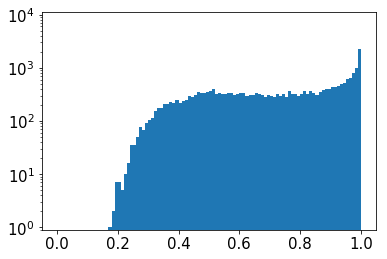}}&
			\raisebox{-.5\height}{\includegraphics[width=0.5\columnwidth]{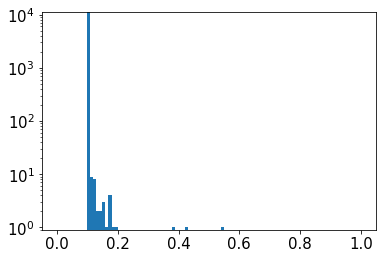}}&
			\raisebox{-.5\height}{\includegraphics[width=0.5\columnwidth]{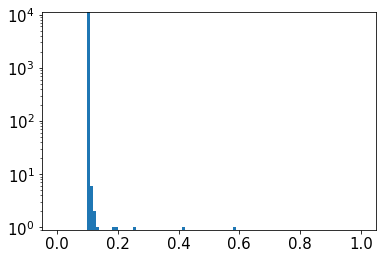}}\\
			\begin{tabular}{c}
				Adversarial \\ Noise
			\end{tabular} &
			\raisebox{-.5\height}{\includegraphics[width=0.5\columnwidth]{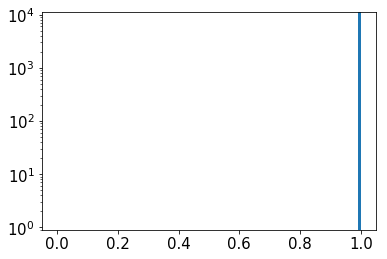}}&
			\raisebox{-.5\height}{\includegraphics[width=0.5\columnwidth]{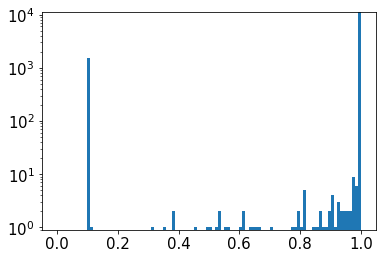}}&
			\raisebox{-.5\height}{\includegraphics[width=0.5\columnwidth]{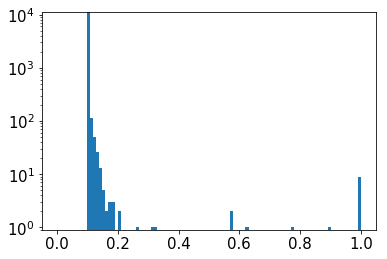}}\\
			\begin{tabular}{c}
				Adversarial \\ Samples
			\end{tabular} &
			\raisebox{-.5\height}{\includegraphics[width=0.5\columnwidth]{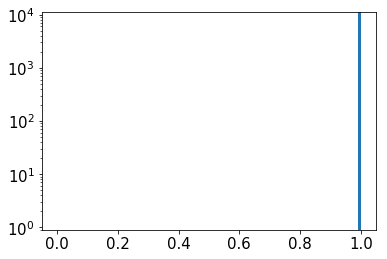}}&
			\raisebox{-.5\height}{\includegraphics[width=0.5\columnwidth]{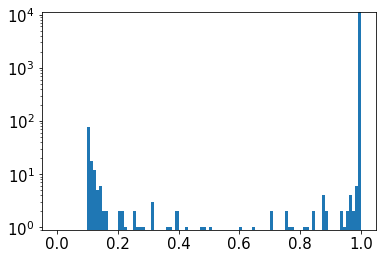}}&
			\raisebox{-.5\height}{\includegraphics[width=0.5\columnwidth]{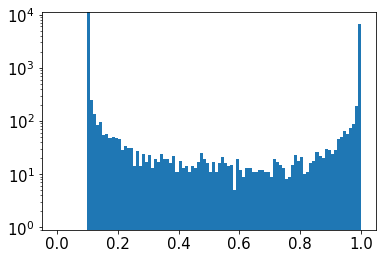}}	
		\end{tabular}
	\end{center}
	\begin{center}
		\caption{\label{hist:svhn}Histograms (logarithmic scale) of maximum confidence values of the three compared models for \textbf{SVHN} on various evaluation datasets. 
		}
	\end{center}
\end{figure*}

\begin{figure*}[ht]
	\begin{center}
		\begin{tabular}[t]{cccc}
			\textbf{Dataset} & \textbf{Plain} & \textbf{CEDA} & \textbf{ACET}\\
			\begin{tabular}{c}
				\textbf{CIFAR-10}
			\end{tabular} &
			\raisebox{-.5\height}{\includegraphics[width=0.5\columnwidth]{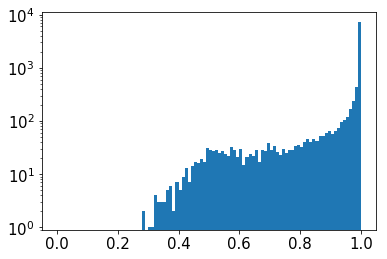}}&
			\raisebox{-.5\height}{\includegraphics[width=0.5\columnwidth]{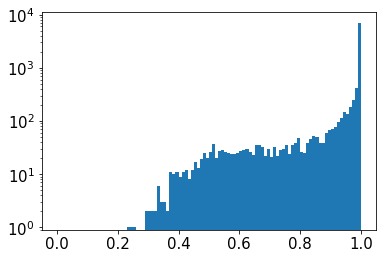}}&
			\raisebox{-.5\height}{\includegraphics[width=0.5\columnwidth]{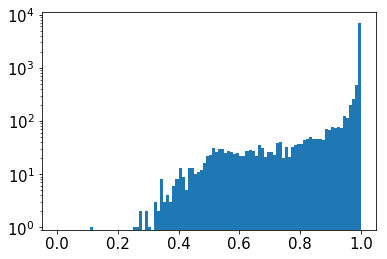}}\\
			\begin{tabular}{c}
				SVHN 
			\end{tabular} &
			\raisebox{-.5\height}{\includegraphics[width=0.5\columnwidth]{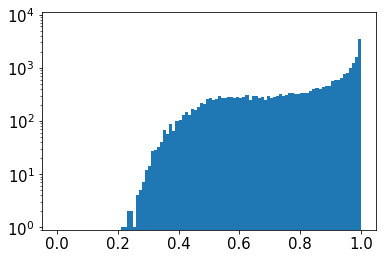}}&
			\raisebox{-.5\height}{\includegraphics[width=0.5\columnwidth]{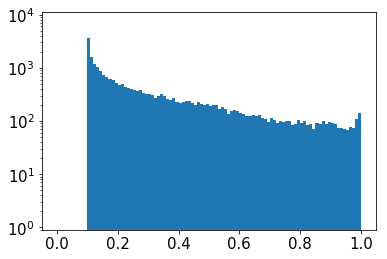}}&
			\raisebox{-.5\height}{\includegraphics[width=0.5\columnwidth]{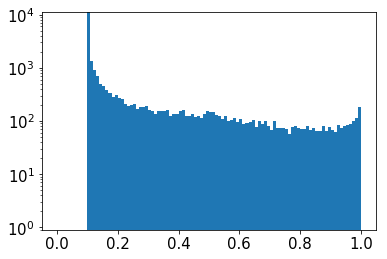}}\\
			\begin{tabular}{c}
				CIFAR-100 
			\end{tabular} &
			\raisebox{-.5\height}{\includegraphics[width=0.5\columnwidth]{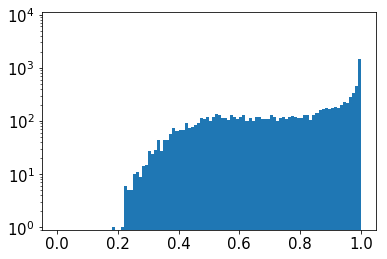}}&
			\raisebox{-.5\height}{\includegraphics[width=0.5\columnwidth]{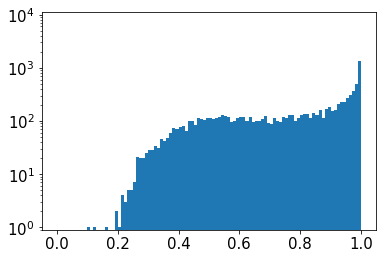}}&
			\raisebox{-.5\height}{\includegraphics[width=0.5\columnwidth]{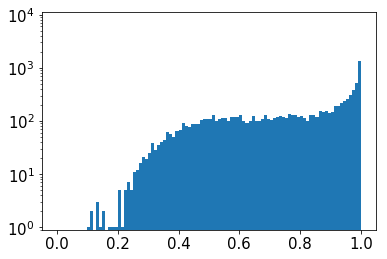}}\\
			\begin{tabular}{c}
				LSUN \\ Classroom
			\end{tabular} &
			\raisebox{-.5\height}{\includegraphics[width=0.5\columnwidth]{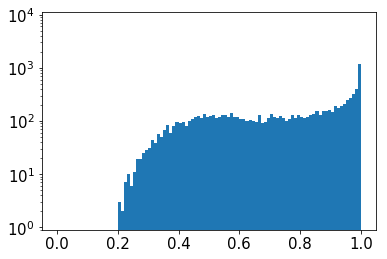}}&
			\raisebox{-.5\height}{\includegraphics[width=0.5\columnwidth]{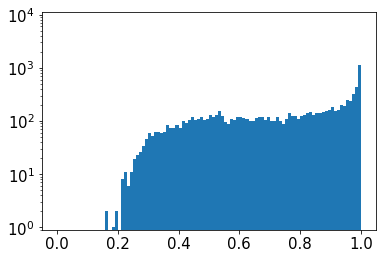}}&
			\raisebox{-.5\height}{\includegraphics[width=0.5\columnwidth]{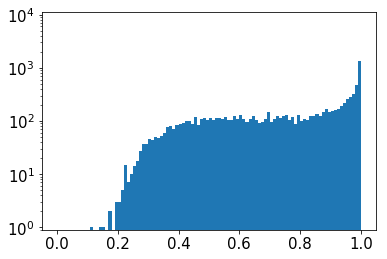}}\\
			\begin{tabular}{c}
				Imagenet \\ minus C10
			\end{tabular} &
			\raisebox{-.5\height}{\includegraphics[width=0.5\columnwidth]{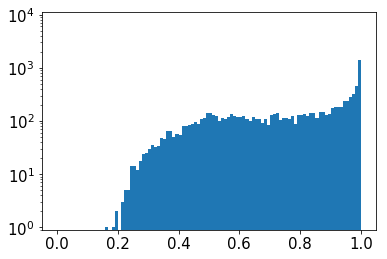}}&
			\raisebox{-.5\height}{\includegraphics[width=0.5\columnwidth]{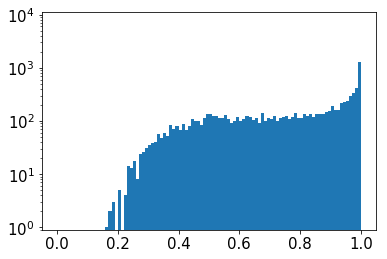}}&
			\raisebox{-.5\height}{\includegraphics[width=0.5\columnwidth]{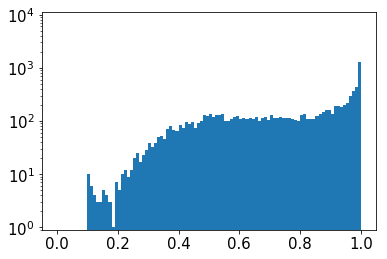}}\\
			\begin{tabular}{c}
				Noise
			\end{tabular} &
			\raisebox{-.5\height}{\includegraphics[width=0.5\columnwidth]{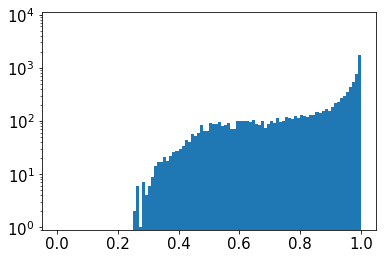}}&
			\raisebox{-.5\height}{\includegraphics[width=0.5\columnwidth]{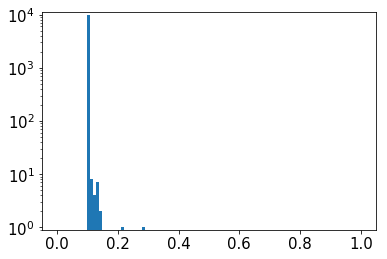}}&
			\raisebox{-.5\height}{\includegraphics[width=0.5\columnwidth]{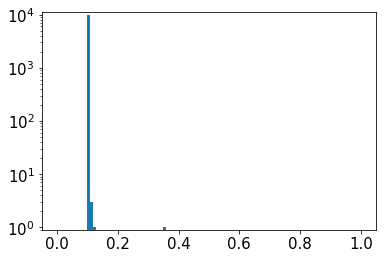}}\\
			\begin{tabular}{c}
				Adversarial \\ Noise
			\end{tabular} &
			\raisebox{-.5\height}{\includegraphics[width=0.5\columnwidth]{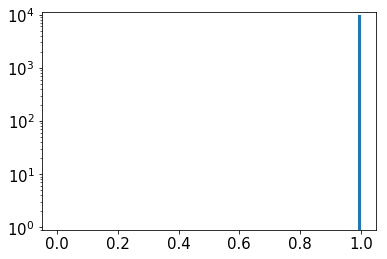}}&
			\raisebox{-.5\height}{\includegraphics[width=0.5\columnwidth]{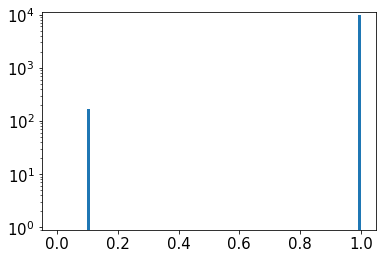}}&
			\raisebox{-.5\height}{\includegraphics[width=0.5\columnwidth]{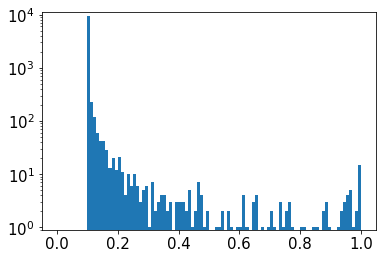}}\\
			\begin{tabular}{c}
				Adversarial \\ Samples
			\end{tabular} &
			\raisebox{-.5\height}{\includegraphics[width=0.5\columnwidth]{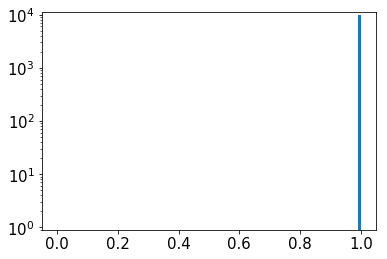}}&
			\raisebox{-.5\height}{\includegraphics[width=0.5\columnwidth]{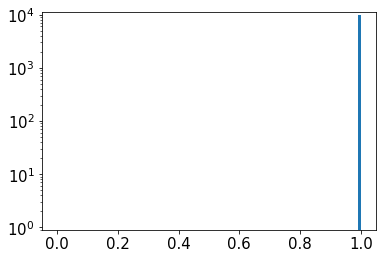}}&
			\raisebox{-.5\height}{\includegraphics[width=0.5\columnwidth]{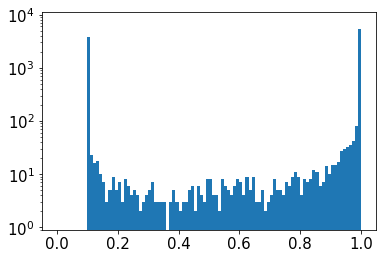}}	
		\end{tabular}
	\end{center}
	\begin{center}
		\caption{\label{hist:cifar10}Histograms (logarithmic scale) of maximum confidence values of the three compared models for \textbf{CIFAR-10} on various evaluation datasets. 
		}
	\end{center}
\end{figure*}

\begin{figure*}[ht]
	\begin{center}
		\begin{tabular}[t]{cccc}
			\textbf{Dataset} & \textbf{Plain} & \textbf{CEDA} & \textbf{ACET}\\
			\begin{tabular}{c}
				\textbf{CIFAR-100}
			\end{tabular} &
			\raisebox{-.5\height}{\includegraphics[width=0.5\columnwidth]{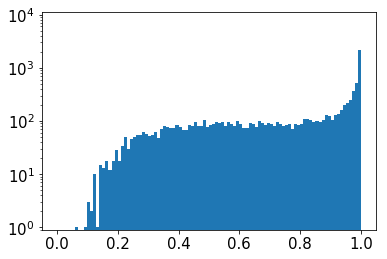}}&
			\raisebox{-.5\height}{\includegraphics[width=0.5\columnwidth]{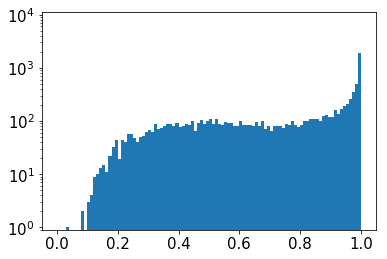}}&
			\raisebox{-.5\height}{\includegraphics[width=0.5\columnwidth]{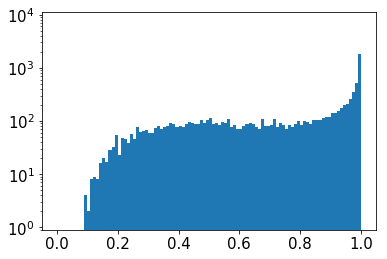}}\\
			\begin{tabular}{c}
				SVHN 
			\end{tabular} &
			\raisebox{-.5\height}{\includegraphics[width=0.5\columnwidth]{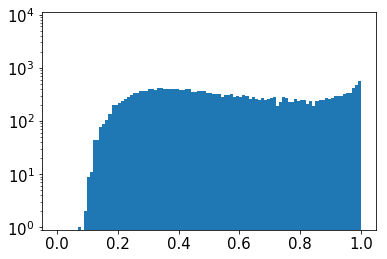}}&
			\raisebox{-.5\height}{\includegraphics[width=0.5\columnwidth]{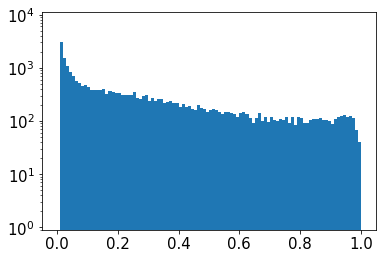}}&
			\raisebox{-.5\height}{\includegraphics[width=0.5\columnwidth]{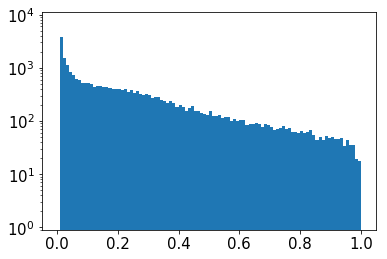}}\\
			\begin{tabular}{c}
				CIFAR-10
			\end{tabular} &
			\raisebox{-.5\height}{\includegraphics[width=0.5\columnwidth]{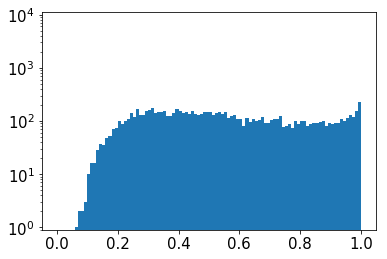}}&
			\raisebox{-.5\height}{\includegraphics[width=0.5\columnwidth]{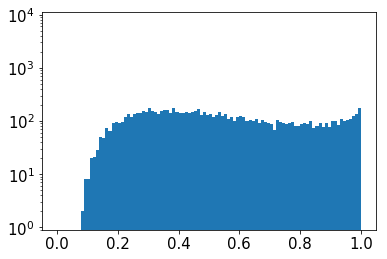}}&
			\raisebox{-.5\height}{\includegraphics[width=0.5\columnwidth]{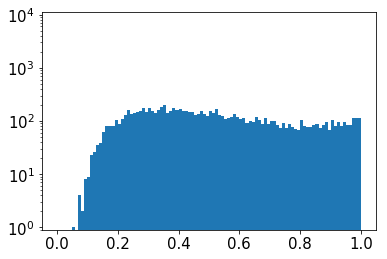}}\\
			\begin{tabular}{c}
				LSUN \\ Classroom
			\end{tabular} &
			\raisebox{-.5\height}{\includegraphics[width=0.5\columnwidth]{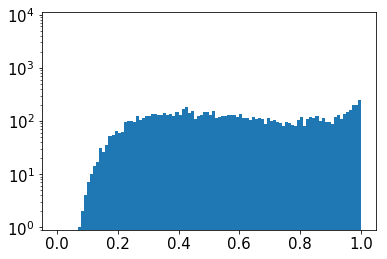}}&
			\raisebox{-.5\height}{\includegraphics[width=0.5\columnwidth]{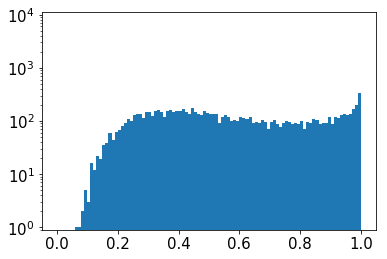}}&
			\raisebox{-.5\height}{\includegraphics[width=0.5\columnwidth]{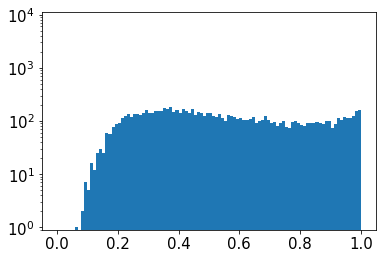}}\\
			\begin{tabular}{c}
				Imagenet \\ minus C10
			\end{tabular} &
			\raisebox{-.5\height}{\includegraphics[width=0.5\columnwidth]{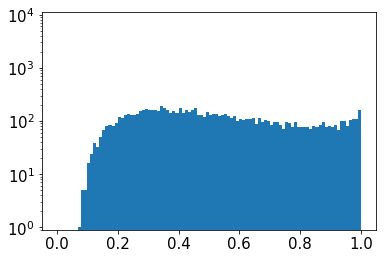}}&
			\raisebox{-.5\height}{\includegraphics[width=0.5\columnwidth]{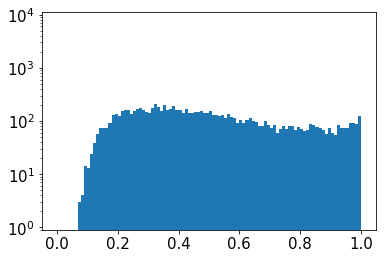}}&
			\raisebox{-.5\height}{\includegraphics[width=0.5\columnwidth]{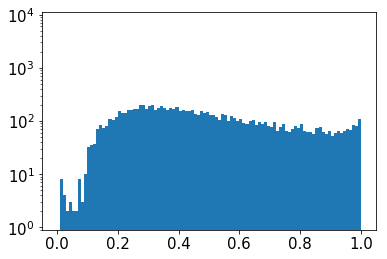}}\\
			\begin{tabular}{c}
				Noise
			\end{tabular} &
			\raisebox{-.5\height}{\includegraphics[width=0.5\columnwidth]{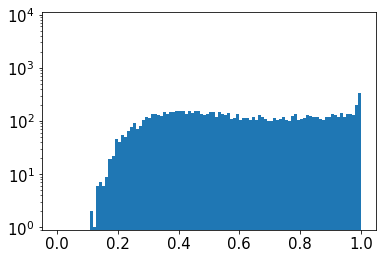}}&
			\raisebox{-.5\height}{\includegraphics[width=0.5\columnwidth]{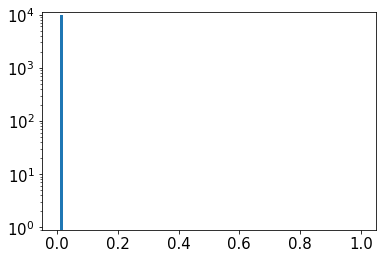}}&
			\raisebox{-.5\height}{\includegraphics[width=0.5\columnwidth]{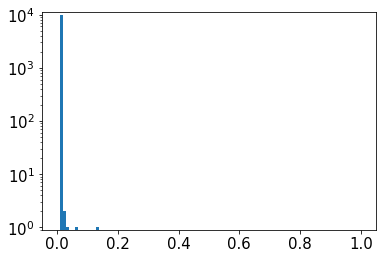}}\\
			\begin{tabular}{c}
				Adversarial \\ Noise
			\end{tabular} &
			\raisebox{-.5\height}{\includegraphics[width=0.5\columnwidth]{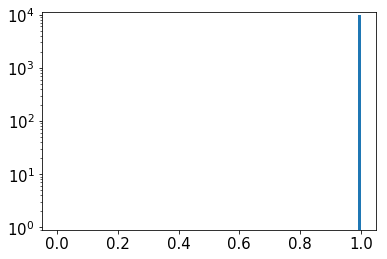}}&
			\raisebox{-.5\height}{\includegraphics[width=0.5\columnwidth]{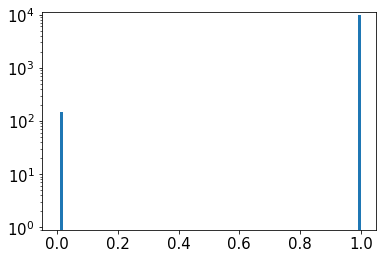}}&
			\raisebox{-.5\height}{\includegraphics[width=0.5\columnwidth]{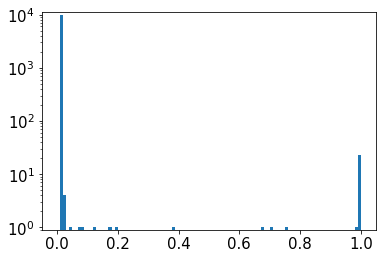}}\\
			\begin{tabular}{c}
				Adversarial \\ Samples
			\end{tabular} &
			\raisebox{-.5\height}{\includegraphics[width=0.5\columnwidth]{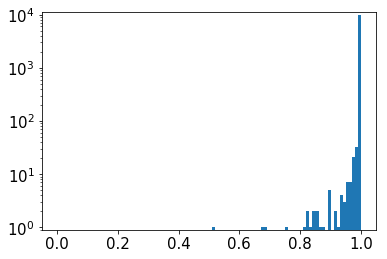}}&
			\raisebox{-.5\height}{\includegraphics[width=0.5\columnwidth]{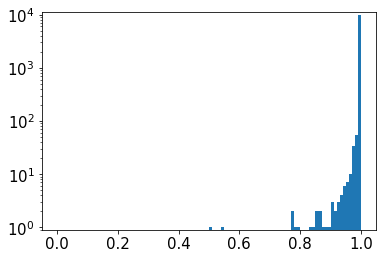}}&
			\raisebox{-.5\height}{\includegraphics[width=0.5\columnwidth]{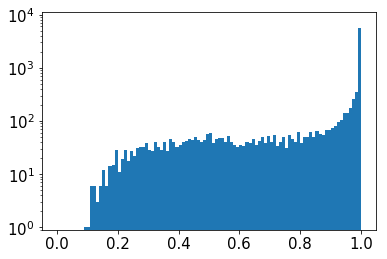}}	
		\end{tabular}
	\end{center}
	\begin{center}
		\caption{\label{hist:cifar100}Histograms (logarithmic scale) of maximum confidence values of the three compared models for \textbf{CIFAR-100} on various evaluation datasets. 
		}
	\end{center}
\end{figure*}

\end{document}